\documentclass[11pt,a4paper]{article}

\usepackage{iftex}
\ifPDFTeX
  \usepackage[utf8]{inputenc}
  \usepackage[T1]{fontenc}
  \usepackage{lmodern}
\fi

\usepackage[margin=1in]{geometry}

\usepackage{amsmath,amssymb,amsfonts,amsthm,mathtools}

\usepackage{microtype}
\usepackage{setspace}

\usepackage{bm}
\usepackage{bbm}

\usepackage{booktabs}
\usepackage{array}
\usepackage{enumitem}
\usepackage{algorithm}
\usepackage{algorithmic}
\usepackage{float}
\usepackage[authoryear,round]{natbib}
\usepackage{xurl}
\usepackage{graphicx}
\usepackage{xcolor}
\usepackage{tikz}
\usetikzlibrary{arrows.meta}

\usepackage[breaklinks=true,colorlinks=true,linkcolor=blue,citecolor=red,urlcolor=blue]{hyperref}
\usepackage[nameinlink,capitalize,noabbrev]{cleveref}
\usepackage{lastpage}
\newtheorem{definition}{Definition}[section]
\newtheorem{assumption}{Assumption}[section]
\newtheorem{proposition}{Proposition}[section]
\newtheorem{lemma}{Lemma}[section]
\newtheorem{theorem}{Theorem}[section]

\newtheorem{remark}{Remark}[section]

\newcommand{\R}{\mathbb{R}}

\newcommand{\jsr}{\rho}

\newcommand{\set}[1]{\left\{ #1 \right\}}
\newcommand{\co}{\operatorname{co}}
\newcommand{\diag}{\operatorname{diag}}

\title{Geometrically Averaged Hard Target Updates for Linear Q-Learning}

\author{Donghwan Lee\\
School of Electrical Engineering, KAIST\\
Daejeon, Republic of Korea\\
\texttt{donghwan@kaist.ac.kr}}

\date{}

\begin{document}

\maketitle

\begin{abstract}
Periodic hard target updates are among the most common stabilization devices in
modern deep Q-learning. Recent studies suggest that target updates can improve
stability in Q-learning with function approximation, including linear function
approximation. We introduce and analyze the so-called $\lambda$-target
update, obtained by averaging the $m$-periodic target update maps with
$\lambda$-geometric weights $(1-\lambda)\lambda^{m-1}$, $\lambda \in [0,1]$. The
endpoint $\lambda=0$ recovers the one-period target update,
while the continuous endpoint $\lambda\uparrow1$ recovers projected Q-value
iteration. We study this mechanism for Q-learning with linear function
approximation, namely linear Q-learning, using a switching-system model
and related tools. For clarity, the paper treats a deterministic version; the
formulation extends to stochastic reinforcement-learning settings.
\end{abstract}


\section{Introduction}
\label{sec:introduction}

In reinforcement learning (RL)~\citep{sutton1998reinforcement}, target
updates~\citep{MnihEtAl2015,LillicrapEtAl2015DDPG} are among the most widely
used stabilization mechanisms for bootstrapping. In deep
Q-learning~\citep{MnihEtAl2015}, a target network is copied from the online
network only periodically; in actor-critic methods, a target network often
tracks the online network by averaging. These mechanisms are empirically robust,
but their precise theoretical effect is still under active study when the
Bellman maximum is combined with function approximation.

A natural question is therefore what target updates change in a tractable
function-approximation model. Existing theory gives several complementary
explanations of target networks. Target-based temporal-difference
learning~\citep{leehe2019target} analyzes separate online and target variables
for policy evaluation, periodic Q-learning~\citep{leehe2020periodic} explains
hard target copies in the tabular setting, and other results obtain stability
with target networks under projections~\citep{zhang2021breaking},
truncation~\citep{chen2023target}, regularization~\citep{zhang2021breaking,limlee2024regularized},
nonlinear regularity assumptions~\citep{fellows2023why}, or over-parameterized
structure~\citep{che2024target}. These results clarify several aspects of the
target-network phenomenon, but they do not directly describe the exact boundary
dynamics of periodic hard target updates in linear Q-learning.

Recent work by~\citet{lee2026targetupdates} studies periodic hard target updates
more explicitly for Q-learning with linear function approximation (linear
Q-learning) and a fixed step-size, using a switching-system
model~\citep{leehuhe2023discrete,LeeLim2026} and the joint spectral radius
(JSR)~\citep{rota1960note,jungers2009joint}. We follow
that paper~\citep{lee2026targetupdates} and extend its periodic
hard-target viewpoint. For clarity, we treat a deterministic version; the same
formulation can be extended to stochastic reinforcement-learning versions.

Before introducing the geometric average, we briefly recall the fixed
$m$-period target update. For a fixed hard-target period $m$, one can freeze the
boundary target, run $m$ online updates, and read off the target-boundary map.
We call this method \emph{$m$-period deterministic linear Q-learning}
($m$-DLQL). The case $m=1$ becomes standard deterministic linear Q-learning,
which we call DLQL\@. It was shown in~\citet{lee2026targetupdates} that, under a
step-size condition, $m$-DLQL becomes projected Q-value iteration (PQVI) as
$m\to\infty$.

We replace the choice of a single integer period by a geometric average over
all period maps. This produces a continuous period parameter: $\lambda=0$ puts
all weight on the period-one linear Q-learning map (DLQL), while
$\lambda\uparrow1$ moves the effective period to infinity and recovers projected
Q-value iteration under the stated relaxation step-size condition. We call the proposed method \emph{$\lambda$ deterministic linear
Q-learning} ($\lambda$-DLQL). This use of a $\lambda$ parameter is similar in
spirit to TD($\lambda$) and Q($\lambda$): those methods use eligibility traces
or $\lambda$-weighted multi-step targets to interpolate between one-step and
longer-horizon updates~\citep{sutton1988learning,watkins1989learning,pengwilliams1996incremental,sutton1998reinforcement}.
The similarity is the geometric interpolation between short- and long-horizon
updates, whereas the difference is that TD($\lambda$) and Q($\lambda$) average
returns or eligibility-trace effects within a single online recursion, while
$\lambda$-DLQL averages hard-target boundary maps and is analyzed through the
resulting switched linear family.

We analyze the properties and convergence of this method using the
switching-system model of linear Q-learning developed recently
by~\citet{LeeLim2026} and the target-update analysis
of~\citet{lee2026targetupdates}. The Bellman maximum is represented by
policy-indexed modes, and the homogeneous error dynamics are studied through the
corresponding switching-system models~\citep{liberzon2003switching,lin2009stability,shorten2007stability}
and their JSRs~\citep{rota1960note,blondel2005computational,jungers2009joint}.
The JSR is the worst-case asymptotic exponential growth rate of all matrix
products generated by a switching family; a value below one certifies uniform
exponential decay under arbitrary switching. The two endpoints solve the same
projected Q-Bellman equation, but their switching-system models are different,
so equality of fixed points does not by itself imply equality of stability
certificates.

The main contribution is a geometrically averaged hard-target mechanism that
connects DLQL and PQVI without introducing a different fixed-point equation. We
show that small $\lambda$ inherits the DLQL endpoint, while $\lambda$ close to
one inherits the PQVI endpoint. Whenever the active endpoint or interior
switching set is JSR-stable, the standard JSR-based Lyapunov construction used
in~\citet{lee2026targetupdates} gives uniform exponential bounds for the
boundary errors and hence convergence to the projected Q-Bellman fixed point.
We also give recursive, inverse-free, and sampled-period implementations of the
same boundary update, so the geometric target mechanism can be written without
explicitly selecting a single hard-target period.

\section{Related Works}
Target networks entered modern deep reinforcement-learning practice through
hard-copy target updates in deep Q-learning and soft target averaging in
actor-critic methods~\citep{MnihEtAl2015,LillicrapEtAl2015DDPG}. In deep
Q-networks~\citep{MnihEtAl2015}, the target network is copied periodically from the online network, and
the Bellman target is held fixed across multiple online updates. In actor-critic
methods~\citep{LillicrapEtAl2015DDPG}, target networks often track the online networks through Polyak-type
averaging~\citep{polyakjuditsky1992acceleration}. This paper does not analyze a deep nonlinear model; rather, it
studies the target-update mechanism in Q-learning with linear function
approximation (linear Q-learning), where the exact boundary dynamics can be
written as switched linear systems.

The closest theoretical line studies target networks and related stabilization
mechanisms. Target-based TD learning analyzes online and target variables for
policy evaluation with linear function approximation~\citep{leehe2019target}.
Periodic Q-learning gives finite-sample guarantees for tabular control with hard
target copies~\citep{leehe2020periodic}. Other analyses stabilize bootstrapping
with target networks together with projections~\citep{zhang2021breaking},
truncation~\citep{chen2023target}, regularization~\citep{zhang2021breaking,limlee2024regularized},
nonlinear regularity assumptions~\citep{fellows2023why}, over-parameterized
structure~\citep{che2024target}, or periodic and soft target dynamics for
linear Q-learning~\citep{lee2026targetupdates}. These works explain several
aspects of target-network behavior, while they do not provide the same
geometrically averaged hard-target boundary map or the same exact JSR
certificate for the unregularized deterministic recursion studied here.

Among these works, the closest theoretical predecessor is
\citet{lee2026targetupdates}. Both papers study target updates for linear
Q-learning through the exact switching-system dynamics induced by the Bellman
maximum, fixed-step boundary/error recursions, and JSR-based stability
certificates. The difference is that \citet{lee2026targetupdates} analyzes fixed
periodic hard target updates and soft target dynamics, whereas this paper
constructs a geometrically averaged hard-period boundary map. This produces a
continuous \(\lambda\) parameter connecting DLQL and PQVI, together with the
endpoint, recursive, inverse-free, and sampled-period analyses developed below.

The projected Bellman equation and linear Q-learning background come from the
standard discounted-MDP and reinforcement learning literature~\citep{puterman2014markov,bertsekas1996neuro,sutton1998reinforcement}.
Classical Q-learning and stochastic-approximation analyses provide the baseline
bootstrapping recursion and its tabular convergence theory~\citep{watkins1992q,tsitsiklis1994asynchronous,jaakkola1994convergence}. With linear
function approximation, projected Bellman equations and projected value-iteration
maps introduce stability questions that are distinct from fixed-point existence
and uniqueness~\citep{meyn2024projected,limlee2025understanding}. The present
paper uses these projected equations as the fixed-point benchmark, but its main
object is the target-induced switching system model.

The stability tools are drawn from switching system theory and JSR
methods~\citep{liberzon2003switching,lin2009stability,shorten2007stability,
rota1960note,blondel2005computational,jungers2009joint}. Recent work applies
this viewpoint directly to Q-learning and linear function approximation by
representing the Bellman maximum as a switching signal and certifying convergence
through JSR bounds~\citep{leehuhe2023discrete,
lee2026lyapunovcertified,LeeLim2026}. The contribution here is to use the same
language for the target update construction developed here: a geometric average
over all hard-target periods, together with a recursive solver and an
inverse-free auxiliary recursion for the resulting boundary update.

\section{Preliminaries and Linear-Approximation Setup}
\label{sec:preliminaries}

\subsection{Notation}
\label{sec:notation}

The set of real numbers is denoted by $\R$; $\R^m$ is the $m$-dimensional
Euclidean space; and $\R^{m\times r}$ is the set of all $m\times r$ real
matrices. For a matrix $A$, $A^\top$ denotes its transpose. The identity
matrix is denoted by $I$. For vectors, $e_i$ is the $i$th standard basis vector,
with dimension clear from context, and $\otimes$ denotes the Kronecker product.
For a finite set $\mathcal S$, $|\mathcal S|$ denotes its cardinality. We write
\[
\Delta_m:=\left\{q\in\R^m:q_i\geq0,\ \sum_{i=1}^m q_i=1\right\}
\]
for the probability simplex in $\R^m$. For a finite matrix family
$\mathcal H=\{A_1,\ldots,A_N\}$,
\[
\co(\mathcal H)
:=\left\{\sum_{i=1}^N\lambda_i A_i:
\lambda_i\geq0,\ \sum_{i=1}^N\lambda_i=1\right\}
\]
denotes its convex hull.
We also use standard matrix notation that appears repeatedly below. For a
vector $x$, $\|x\|_2$ is the Euclidean norm. For a square matrix $A$,
$\rho(A)$ denotes its ordinary spectral radius, while $\rho(\mathcal H)$ denotes
the joint spectral radius of a switching family once the JSR is defined below.
For a matrix $B$, $\operatorname{range}(B)$ denotes its column space,
$B\succ0$ means that $B$ is symmetric positive definite, and
$\lambda_{\max}(B)$ denotes the largest eigenvalue when $B$ is symmetric.
Expectations are denoted by $\mathbb E[\cdot]$.

\subsection{Switched Linear Systems}
\label{sec:switching_systems}

The stability certificates used later are stated in the language of switched
systems, so we first recall the basic model before specializing it to the
Bellman-induced switching families. Let us consider the discrete-time switched
affine system~\citep{liberzon2003switching,lin2009stability,shorten2007stability}
\[
  x_{k+1}=A_{\sigma_k}x_k+b_{\sigma_k},
\]
where each index $i\in\{1,2,\ldots,M\}$, equivalently each affine pair
$(A_i,b_i)$, is called a \emph{mode}, and $\sigma_k$ is the switching signal that
selects the active mode at time $k$. The matrix $A_{\sigma_k}$ is selected from
the prescribed family $\mathcal H:=\{A_1,A_2,\ldots,A_M\}$, which is called a
\emph{switching family}; $b_{\sigma_k}$ is a mode-dependent affine term. When
$b_{\sigma_k}=0$, this reduces to a switched linear system,
$x_{k+1}=A_{\sigma_k}x_k$. The worst-case exponential
rate of the switched linear family is characterized by the \emph{joint spectral radius}
(JSR)~\citep{rota1960note,blondel2005computational,jungers2009joint},
defined as follows.
\begin{definition}
\label{def:jsr}
For a bounded set of matrices $\mathcal H\subset\R^{m\times m}$, its joint
spectral radius is
\[
\jsr(\mathcal H)
:=
\lim_{k\to\infty}
\sup_{A_1,\ldots,A_k\in\mathcal H}
\|A_k\cdots A_1\|^{1/k}.
\]
\end{definition}
We note that the JSR is independent of the chosen submultiplicative
norm~\citep{rota1960note,jungers2009joint}. When $\mathcal H$ is finite, the
supremum for each fixed product length is a maximum over products generated by
matrices in $\mathcal H$. For a finite family $\mathcal H$, the notation
$\jsr(\co(\mathcal H))$ means the JSR computed when each factor in a product is
allowed to be any convex combination of matrices in $\mathcal H$.
Throughout the later JSR certificates, $\rho(\mathcal H)$ denotes this same JSR
value when the argument is a switching family.

\subsection{Joint Spectral Radius and Lyapunov Certificates}
\label{sec:joint_spectral_radius}

The JSR in~\Cref{def:jsr} turns arbitrary switched products into a
single worst-case exponential rate. A switched linear system is uniformly
exponentially stable under arbitrary switching if there exist constants $C\geq1$
and $\eta\in(0,1)$ such that
\[
\|A_{\sigma_{k-1}}\cdots A_{\sigma_0}x\|_2\leq C\eta^k\|x\|_2
\]
for every horizon $k\geq0$, every initial state $x\in\R^m$, and every switching
sequence. A \emph{common Lyapunov function} for $\mathcal H$ is a positive definite
function that decreases along every mode. In the analysis below, the Bellman
maximum in linear Q-learning induces stochastic-policy switching, and the
Lyapunov functions are built from products of the corresponding mode matrices.
The following finite-family piecewise-quadratic construction~\citep{hushenzhang2010generating,lee2026lyapunovcertified} is the Lyapunov
certificate used in the deterministic, stochastic, and target-network arguments.
\begin{lemma}
\label{lem:common_lyapunov_construction}
Let
\[
\mathcal H=\{A_1,A_2,\ldots,A_M\}\subset\R^{m\times m}
\]
and suppose that \(\rho(\mathcal H)<1\). Fix \(\epsilon>0\) such that
\(\beta_\epsilon:=\rho(\mathcal H)+\epsilon<1\). Then there exist a norm \(p_\epsilon\)
on \(\R^m\) and a constant \(C_\epsilon\geq1\) such that
\[
\|x\|_2\leq p_\epsilon(x)\leq \sqrt{C_\epsilon}\|x\|_2,
\qquad x\in\R^m,
\]
and
\[
p_\epsilon(Bx)\leq \beta_\epsilon p_\epsilon(x),
\qquad x\in\R^m,
\qquad B\in\co(\mathcal H).
\]
Consequently, any switched recursion \(z_{k+1}=B_kz_k\) with
\(B_k\in\co(\mathcal H)\) satisfies
\[
p_\epsilon(z_k)\leq\beta_\epsilon^k p_\epsilon(z_0),
\qquad
\|z_k\|_2\leq \sqrt{C_\epsilon}\,\beta_\epsilon^k\|z_0\|_2,
\qquad k\geq0.
\]
In particular, \(z_k\to0\) for every initial condition.
\end{lemma}
This is the only convergence lemma used below. Once an algorithm has been
reduced to a switched error recursion whose active matrices lie in the convex
hull of a deterministic family with JSR less than one, the displayed estimates
apply directly.

The next two elementary matrix facts are used when summing the geometrically
weighted hard-period maps and when passing to the endpoint \(\lambda\uparrow1\).
They are stated here in a general form so that the later proofs do not repeat the
same series arguments.
\begin{lemma}
\label{lem:geometric_resolvent_identity}
Let \(A\in\R^{d\times d}\) satisfy \(\rho(A)<1\). Then, for every
\(0\leq\lambda<1\), the matrix \(I-\lambda A\) is invertible and
\[
(1-\lambda)\sum_{m=1}^{\infty}\lambda^{m-1}A^m
=
(1-\lambda)A(I-\lambda A)^{-1}.
\]
\end{lemma}
\begin{proof}
Since \(\rho(A)<1\) and \(0\leq\lambda<1\), we have \(\rho(\lambda A)<1\).
Hence the Neumann series converges and gives
\[
(I-\lambda A)^{-1}=\sum_{j=0}^{\infty}(\lambda A)^j
\]
in any submultiplicative matrix norm; see, for example,
\citet[Section~1.3, Theorem~1.4]{varga2000matrix}. Multiplying the series by \((1-\lambda)A\) gives
\[
(1-\lambda)A(I-\lambda A)^{-1}
=(1-\lambda)\sum_{j=0}^{\infty}\lambda^jA^{j+1}
=(1-\lambda)\sum_{m=1}^{\infty}\lambda^{m-1}A^m.
\]
\end{proof}

\begin{lemma}
\label{lem:abel_convergence_matrix_sequences}
Let \((X_m)_{m\geq1}\) be a sequence of matrices in a finite-dimensional matrix
space, and suppose that \(X_m\to X\). Then
\[
\lim_{\lambda\uparrow1}
(1-\lambda)\sum_{m=1}^{\infty}\lambda^{m-1}X_m
=X.
\]
\end{lemma}
\begin{proof}
Since \(X_m\to X\), the sequence is
bounded, and the displayed series is well defined for \(0\leq\lambda<1\). Fix
\(\varepsilon>0\). Choose \(N\) such that \(\|X_m-X\|\leq\varepsilon\) for all
\(m\geq N\). Then
\[
\begin{aligned}
&\left\|(1-\lambda)\sum_{m=1}^{\infty}\lambda^{m-1}X_m-X\right\| \\
&\quad\leq
(1-\lambda)\sum_{m=1}^{N-1}\lambda^{m-1}\|X_m-X\|
+(1-\lambda)\sum_{m=N}^{\infty}\lambda^{m-1}\varepsilon.
\end{aligned}
\]
The first term goes to zero as \(\lambda\uparrow1\), because it is a finite sum
multiplied by \(1-\lambda\). The second term is at most \(\varepsilon\). Since
\(\varepsilon\) is arbitrary, the limit follows.
\end{proof}

\subsection{Discounted MDPs with Linear Function Approximation}
\label{sec:discounted_mdps_lfa}

We consider a finite discounted Markov decision process (MDP)~\citep{puterman2014markov,bertsekas1996neuro} with state space
$\mathcal S=\{1,\ldots,|\mathcal S|\}$, action space
$\mathcal A=\{1,\ldots,|\mathcal A|\}$, transition probability
$P(s'\mid s,a)$, real-valued one-step reward $r(s,a,s')$, expected reward
\[
R(s,a):=\sum_{s'\in\mathcal S}P(s'\mid s,a)r(s,a,s'),
\]
and discount factor $\gamma\in(0,1)$. State-action functions are viewed as
vectors in $\R^{|\mathcal S||\mathcal A|}$ using the action-block ordering $(1,1),(2,1),\ldots,(|\mathcal S|,1),
(1,2),(2,2),\ldots,(|\mathcal S|,|\mathcal A|)$. All matrices and vectors indexed by state-action pairs use this ordering. Define
\[
P:=
\begin{bmatrix}
P_1\\
\vdots\\
P_{|\mathcal A|}
\end{bmatrix}
\in\R^{|\mathcal S||\mathcal A|\times |\mathcal S|},
\qquad
R:=
\begin{bmatrix}
R(\cdot,1)\\
\vdots\\
R(\cdot,|\mathcal A|)
\end{bmatrix}
\in\R^{|\mathcal S||\mathcal A|},
\]
where $P_a=P(\cdot\mid\cdot,a)\in\R^{|\mathcal S|\times|\mathcal S|}$.
Let $\Theta$ denote the set of deterministic stationary policies
$\pi:\mathcal S\to\mathcal A$. For any stochastic policy $\mu$ with $\mu(s)\in\Delta_{|\mathcal A|}$ for each
$s\in\mathcal S$, define
\[
\Pi^\mu:=
\begin{bmatrix}
\mu(1)^\top\otimes e_1^\top\\
\mu(2)^\top\otimes e_2^\top\\
\vdots\\
\mu(|\mathcal S|)^\top\otimes e_{|\mathcal S|}^\top
\end{bmatrix}
\in\R^{|\mathcal S|\times|\mathcal S||\mathcal A|}.
\]
For a deterministic policy $\pi\in\Theta$, the same notation $\Pi^\pi$ is used by
identifying $\pi(s)$ with its one-hot encoding.
For
$Q\in\R^{|\mathcal S||\mathcal A|}$, define
\[
V_Q(s):=\max_{a\in\mathcal A}Q(s,a),
\qquad
V_Q:=(V_Q(1),\ldots,V_Q(|\mathcal S|))^\top.
\]
The Bellman optimality operator is $F(Q):=R+\gamma P V_Q$. Let $\Phi\in\R^{|\mathcal S||\mathcal A|\times m}$ be a feature matrix. Its
row corresponding to $(s,a)$ is $\phi(s,a)^\top$, where
$\phi(s,a)\in\R^m$. The linear function approximation
(LFA) of the Q-function is $Q_\theta:=\Phi\theta$. For an LFA parameter $\theta$, define the corresponding greedy value vector by
\[
V_\theta(s):=\max_{a\in\mathcal A}\phi(s,a)^\top\theta,
\qquad
V_\theta:=(V_\theta(1),\ldots,V_\theta(|\mathcal S|))^\top.
\]
We use $d$ to denote a state-action sampling distribution on
$\mathcal S\times\mathcal A$. In the i.i.d.\ observation model, $d$ is the
sampling distribution of $(s_k,a_k)$; in the Markovian observation model, $d$ is
the stationary state-action distribution of the behavior-induced chain. Define
\[
D:=\diag(d(s,a))_{(s,a)\in\mathcal S\times\mathcal A}.
\]
The feature and sampling conditions used throughout the rest of the paper are
collected in the following standing assumption. They are the minimal structural
conditions needed to make the projected residual and Gram matrix well defined in
the coordinates used below.
\begin{assumption}
\label{ass:feature_sampling}
The feature matrix $\Phi$ has full column rank, and the sampling distribution
has full support: $d(s,a)>0$ for every
$(s,a)\in\mathcal S\times\mathcal A$. Equivalently, the diagonal sampling
matrix satisfies $D\succ0$.
\end{assumption}
The standing assumption immediately implies
\(\Phi^\top D\Phi\succ0\), because \(D\succ0\) and \(\Phi\) has full column
rank. We use this Gram-matrix positivity throughout without introducing it as a
separate lemma.
The switching-system model above will be used to analyze Bellman updates by
viewing each realization of the maximum operator as a policy-selected mode. To
make this passage from the Bellman maximum to a switching mode precise, we use
one final setup fact. It says that the difference between two greedy value
vectors is itself a value vector generated by a suitable stochastic policy
applied to the parameter difference.
\setcounter{lemma}{3}
\begin{lemma}
\label{lem:stochastic_policy_linearization}
For every pair $\theta,\bar\theta\in\R^m$, there exists a stochastic policy
$\mu_{\theta,\bar\theta}$ such that
\begin{align}
V_\theta-V_{\bar\theta}
=
\Pi^{\mu_{\theta,\bar\theta}}\Phi(\theta-\bar\theta).
\label{eq:stoch_policy_linearization}
\end{align}
\end{lemma}
The proof is the same as the stochastic-policy linearization argument in the
recent switching-system analysis of linear Q-learning by \citet{LeeLim2026}, so
we omit it here.

\section{PQVI as the Fixed-Point Benchmark}
\label{sec:projected_qvi}

PQVI is the fixed-point benchmark for the target
updates studied later. The statements in this section
are recalled from~\citet{lee2026targetupdates}; their proofs are omitted here. By~\Cref{ass:feature_sampling}, $\Phi^\top D\Phi$ is
nonsingular, so the $D$-orthogonal projection onto $\operatorname{range}(\Phi)$ is
\[
\Pi_D:=\Phi(\Phi^\top D\Phi)^{-1}\Phi^\top D.
\]
The projected Bellman equation~\citep{bertsekas1996neuro,meyn2024projected,limlee2025understanding} is
\begin{align}
\Phi\theta
=\Pi_D F(\Phi\theta)
=\Pi_D\left(R+\gamma P V_\theta\right).
\label{eq:prelim_projected_bellman_equation}
\end{align}
A projected Bellman fixed point is any parameter satisfying this equation. Such a
point need not exist or be unique in general~\citep{limlee2025understanding}, so
the analysis fixes the benchmark solution through the following standing
assumption.
\begin{assumption}
\label{ass:unique_projected_q_bellman}
The projected Bellman equation in~\Cref{eq:prelim_projected_bellman_equation} has a
unique solution $\theta^\star$.
\end{assumption}

For later comparison with DLQL, we write PQVI as the exact least-squares
solve with the Bellman target frozen at the current parameter. For a fixed target
parameter $\bar\theta$, define
\[
f(\theta ;\bar \theta )
:=
\left\| {R+\gamma P V_{\bar\theta}-\Phi\theta} \right\|_D^2,
\]
where $\|x\|_D^2:=x^\top D x$. Minimizing $f(\theta;\bar\theta)$ over $\theta$
gives
\[
\bar\theta^*
=(\Phi^\top D\Phi)^{-1}\Phi^\top D(R+\gamma P V_{\bar\theta}),
\qquad
\Phi\bar\theta^*=\Pi_D(R+\gamma P V_{\bar\theta}).
\]
Choosing $\bar\theta=\theta_k$ gives the PQVI update
\begin{align}
\theta_{k+1}
=
(\Phi^\top D\Phi)^{-1}\Phi^\top D
\left(R+\gamma P V_{\theta_k}\right),
\qquad
k\in\{0,1,\ldots\}.
\label{eq:pqvi_map_intro}
\end{align}
Its fixed points satisfy the parameter form of the projected Q-Bellman equation,
\begin{align}
\Phi^\top D\Phi\,\theta=\Phi^\top D\left(R+\gamma P V_\theta\right).
\label{eq:projected_q_bellman_equation}
\end{align}

The fixed-point equation identifies the benchmark solution. To analyze the
iteration itself, we pass to the affine switching-system representation generated
by the Bellman maximum. For each PQVI iterate, choose a deterministic
greedy policy $\pi_k$ satisfying
\[
\pi_k(s)\in\arg\max_{a\in\mathcal A}\phi(s,a)^\top\theta_k,
\qquad s\in\mathcal S,
\]
with fixed deterministic tie breaking. Then $V_{\theta_k}=\Pi^{\pi_k}\Phi\theta_k$,
and direct substitution gives
\[
\theta_{k+1}
=
(\Phi^\top D\Phi)^{-1}\Phi^\top D R
+
(\Phi^\top D\Phi)^{-1}\gamma\Phi^\top DP\Pi^{\pi_k}\Phi\,\theta_k.
\]
The next lemma presents the resulting error recursion. After subtracting the
projected Bellman fixed point $\theta^\star$, the common affine term cancels, and
the stochastic policy comes from the value-difference linearization.
\setcounter{lemma}{1}
\begin{lemma}
\label{lem:pqvi_stochastic_policy_error_recursion}
For each PQVI iterate, there exists a stochastic policy
$\mu_k=\mu_{\theta_k,\theta^\star}$ such that
\begin{align}
\theta_{k+1}-\theta^\star
=A_{\mu_k}^{\mathrm{PQVI}}(\theta_k-\theta^\star),\label{eq:1}
\end{align}
where for a stochastic policy $\mu$ with $\mu(s)\in\Delta_{|\mathcal A|}$,
\begin{align}
A_\mu^{\mathrm{PQVI}}
:=
(\Phi^\top D\Phi)^{-1}\gamma\Phi^\top DP\Pi^\mu\Phi.
\label{eq:pqvi_A_mu_definition}
\end{align}
For a deterministic policy $\pi$, we write $A_\pi^{\mathrm{PQVI}}$.
\end{lemma}

The proof is given in~\citet{lee2026targetupdates}; we omit it here.

For deterministic policies, the finite PQVI switching family is
\begin{align}
\mathcal A^{\mathrm{PQVI}}
:=
\set{A_\pi^{\mathrm{PQVI}}:\pi\in\Theta}.
\label{eq:pqvi_family_intro}
\end{align}
\begin{lemma}
\label{lem:pqvi_jsr_convergence_certificate}
If $\rho(\mathcal A^{\mathrm{PQVI}})<1$, then the PQVI error recursion with
$z_k=\theta_k-\theta^\star$ converges to zero for every initial condition, with
the finite-time bounds given by~\Cref{lem:common_lyapunov_construction}.
\end{lemma}

The proof is given in~\citet{lee2026targetupdates}; we omit it here.

\section{Deterministic Linear Q-Learning (DLQL) and Its Switching Certificate}
\label{sec:deterministic_linear_q_learning}

This section presents the DLQL fixed-point and switching facts needed later. The
supporting lemmas in this section are recalled from~\citet{lee2026targetupdates};
their proofs are omitted in this paper. This section provides the deterministic linear Q-learning (DLQL) update that serves as the period-one
endpoint for the target-period constructions below. Given a transition sample
$(s_k,a_k,r_{k+1},s'_k)$, where $r_{k+1}=r(s_k,a_k,s'_k)$, constant step-size
linear Q-learning~\citep{watkins1992q,tsitsiklis1994asynchronous,jaakkola1994convergence,sutton1998reinforcement} uses
\begin{equation*}
\theta_{k+1}
=\theta_k+\alpha\phi(s_k,a_k)
\left(r_{k+1}+\gamma\max_{u\in\mathcal A}\phi(s'_k,u)^\top\theta_k
-\phi(s_k,a_k)^\top\theta_k\right),
\end{equation*}
where $\alpha>0$ is the step-size. Its deterministic averaged form is
\begin{align}
\theta_{k+1}
=\theta_k+\alpha\Phi^\top D\left(R+\gamma P V_{\theta_k}-\Phi\theta_k\right),
\qquad
k\in\{0,1,\ldots\}.
\label{eq:prelim_deterministic_linear_q_learning}
\end{align}
Using the residual $f$ from~\Cref{sec:projected_qvi} with the target argument
frozen at $\theta_k$, this update is
\[
\theta_{k+1}
=
\theta_k-\frac{\alpha}{2}
\left.\nabla_\theta f(\theta;\theta_k)\right|_{\theta=\theta_k}.
\]
Therefore, this DLQL performs one residual-gradient step toward the frozen Bellman target and
then refreshes the target at the next step.

To connect this update to the projected Bellman equation, let us define the DLQL map
\begin{equation*}
h(\theta):=\theta+\alpha g(\theta),
\end{equation*}
where
\begin{equation*}
g(\theta):=\Phi^\top D\left(R+\gamma P V_\theta-\Phi\theta\right)
\end{equation*}
is the projected Bellman residual. Then,~\Cref{eq:prelim_deterministic_linear_q_learning} is $\theta_{k+1}=h(\theta_k)$.
A fixed point of $h$ is a zero of $g$. The next lemma makes explicit that this
zero set is exactly the projected Bellman solution set, so DLQL and PQVI share the same fixed-point equation.
\begin{lemma}
\label{lem:prelim_pbe_residual_equivalence}
The solution set of $g(\theta)=0$ is identical to the solution set of the
projected Bellman equation in~\Cref{eq:prelim_projected_bellman_equation}.
\end{lemma}

Consequently, the benchmark parameter satisfies
\[
\Phi^\top D\left(R+\gamma P V_{\theta^\star}-\Phi\theta^\star\right)
=g(\theta^\star)=0.
\]
The remaining difference between PQVI and DLQL is therefore dynamic, not
fixed-point based. We now present the DLQL switching model used for the stability certificate. If
$\pi_k$ is the greedy policy selected by $\theta_k$, substitution of
$V_{\theta_k}=\Pi^{\pi_k}\Phi\theta_k$ into~\Cref{eq:prelim_deterministic_linear_q_learning} gives
\[
\theta_{k+1}
=
\alpha\Phi^\top D R
+
\left(I-\alpha\Phi^\top D\Phi
+\alpha\gamma\Phi^\top DP\Pi^{\pi_k}\Phi\right)\theta_k.
\]
The next lemma subtracts the projected Bellman fixed point and uses the same
stochastic-policy linearization as in~\Cref{sec:projected_qvi} to obtain the exact switching linear system model of the error.
\setcounter{lemma}{2}
\begin{lemma}
\label{lem:direct_stochastic_policy_error_recursion}
For each DLQL iterate, there exists a stochastic policy
$\mu_k=\mu_{\theta_k,\theta^\star}$ such that the DLQL error recursion is
represented exactly as
\[
\theta_{k+1}-\theta^\star
=A_{\mu_k}^{\mathrm{DLQL}}(\theta_k-\theta^\star),
\]
where for a stochastic policy $\mu$ with $\mu(s)\in\Delta_{|\mathcal A|}$,
\begin{align}
A_\mu^{\mathrm{DLQL}}:=I-\alpha\Phi^\top D\Phi+\alpha\gamma\Phi^\top DP\Pi^\mu\Phi.
\label{eq:A_mu_definition}
\end{align}
\end{lemma}
The corresponding finite deterministic switching family is
\begin{align}
\mathcal A^{\mathrm{DLQL}}
:=
\set{A_\pi^{\mathrm{DLQL}}: \pi\in\Theta}.
\label{eq:direct_family}
\end{align}
The family in~\Cref{eq:direct_family} is the DLQL stability object. The next
result is the DLQL counterpart of the PQVI certificate in~\Cref{lem:pqvi_jsr_convergence_certificate}: once the JSR of this finite family
is strictly below one, the common Lyapunov bound from~\Cref{lem:common_lyapunov_construction} gives convergence of every DLQL error
trajectory.
\begin{lemma}
\label{lem:dlql_jsr_convergence_certificate}
If $\rho(\mathcal A^{\mathrm{DLQL}})<1$, then the DLQL error recursion with
$z_k=\theta_k-\theta^\star$ converges to zero for every initial condition, with
the finite-time bounds given by~\Cref{lem:common_lyapunov_construction}.
\end{lemma}

Although DLQL and PQVI solve the same projected Q-Bellman equation, their
mode families differ. DLQL can be written as a residual step toward the PQVI update evaluated at the same parameter,
\[
h(\theta_k)
=
\theta_k+
\alpha\Phi^\top D\Phi
\left(
(\Phi^\top D\Phi)^{-1}\Phi^\top D
\left(R+\gamma P V_{\theta_k}\right)
-\theta_k
\right),
\]
and the two iterations coincide only when $\alpha\Phi^\top D\Phi=I$. Hence the DLQL
family $\mathcal A^{\mathrm{DLQL}}$ and the PQVI family
$\mathcal A^{\mathrm{PQVI}}$ require separate stability certificates.

\section{Hard-Target Period Maps Used by \texorpdfstring{$\lambda$}{lambda}-DLQL}
\label{sec:hard_target_period_maps_for_lambda}

The $\lambda$-DLQL construction presented in this paper averages over hard-target periods, so we first
summarize the hard-target period-map results of~\citet{lee2026targetupdates} in
the form needed for the geometric averaging below. The statements in this section
are recalled from~\citet{lee2026targetupdates}; their proofs are omitted here. The periodic hard target update used in deep Q-learning
by~\citet{MnihEtAl2015} holds the target parameter fixed for several online
updates and then copies the online parameter into the target. Here this mechanism
is used to define the boundary maps that will be geometrically averaged in the
next section.

A \emph{target boundary} is an update time at which the target parameter is reset
before a new target block begins. The \emph{boundary target} is the value of
$\bar\theta$ at that boundary. Starting from
$\theta_{t,0}:=\bar\theta_t=\theta_t$, a hard-target period of length $m$ freezes
$\bar\theta_t$ and applies
\[
\theta_{t,i+1}
=
\theta_{t,i}+
\alpha\Phi^\top D\left(R+\gamma P V_{\bar\theta_t}-\Phi\theta_{t,i}\right),
\qquad i=0,\ldots,m-1.
\]
After these $m$ online updates, the next boundary target is copied as
$\bar\theta_{t+m}=\theta_{t,m}$. Equivalently, each within-period update is a
residual-gradient step on $f(\cdot;\bar\theta_t)$ while the target argument is
frozen.

The frozen target makes the policy mode fixed throughout the period. The next
lemma states the boundary map used by \(\lambda\)-DLQL and provides the
within-period identity and its sum.
\setcounter{lemma}{1}
\begin{lemma}
\label{lem:hard_period_boundary_map}
If a hard-target period of length \(m\) starts from \(\bar\theta_t=\theta_t\) and
freezes \(\bar\theta_t\), then there exists a stochastic policy
\(\mu_t=\mu_{\bar\theta_t,\theta^\star}\), fixed throughout the period, such that,
after the hard copy \(\bar\theta_{t+m}=\theta_{t+m}\),
\begin{align}
\theta_{t+m}-\theta^\star=A_{\mu_t,m}^{\mathrm{DLQL}}(\theta_t-\theta^\star),
\qquad
\bar\theta_{t+m}-\theta^\star=\theta_{t+m}-\theta^\star,
\label{eq:hard_stochastic_period_map}
\end{align}
where, for a stochastic policy \(\mu\) with \(\mu(s)\in\Delta_{|\mathcal A|}\) and period
length \(m\geq1\),
\begin{align}
A_{\mu,m}^{\mathrm{DLQL}}
:={}&(I-\alpha\Phi^\top D\Phi)^m +\sum_{i=0}^{m-1}(I-\alpha\Phi^\top D\Phi)^i
\alpha\gamma\Phi^\top DP\Pi^\mu\Phi.
\label{eq:m_dlql_A_mu_m_definition}
\end{align}
\end{lemma}
For deterministic policies, the same formula gives the $m$-DLQL modes
$A_{\pi,m}^{\mathrm{DLQL}}$. The corresponding finite switching family is
\begin{align}
\mathcal A_m^{\mathrm{DLQL}}
:=
\set{A_{\pi,m}^{\mathrm{DLQL}}:\pi\in\Theta}.
\label{eq:m_dlql_family}
\end{align}
These modes appear in the geometric average. Before taking the endpoint limit, we
present the algebraic decomposition of each finite hard-period map.
\setcounter{lemma}{3}
\begin{lemma}
\label{lem:period_map_decomposition}
For every stochastic policy \(\mu\) and every integer \(m\geq1\),
\begin{align}
A_{\mu,m}^{\mathrm{DLQL}}
&=(I-\alpha\Phi^\top D\Phi)^m
+\left[I-(I-\alpha\Phi^\top D\Phi)^m\right]A_\mu^{\mathrm{PQVI}} \nonumber\\
&=A_\mu^{\mathrm{PQVI}}+(I-\alpha\Phi^\top D\Phi)^m(I-A_\mu^{\mathrm{PQVI}}),
\label{eq:m_dlql_converges_to_pqvi}
\end{align}
where \(A_\mu^{\mathrm{PQVI}}=(\Phi^\top D\Phi)^{-1}\gamma\Phi^\top DP\Pi^\mu\Phi\).
In particular, for a deterministic policy \(\pi\), the same identity holds with
\(\mu=\pi\), and \(A_{\pi,1}^{\mathrm{DLQL}}=A_\pi^{\mathrm{DLQL}}\).
\end{lemma}
The proof is given by the hard-period decomposition in~\citet{lee2026targetupdates};
we use the statement here without repeating the finite telescoping argument.
The decomposition separates the PQVI component from the residual power of
\(I-\alpha\Phi^\top D\Phi\). We next use it to state the period-to-PQVI
endpoint convergence and the corresponding JSR convergence of the finite
switching families, which are used later in the \(\lambda\uparrow1\) arguments.
\setcounter{lemma}{2}
\begin{lemma}
\label{lem:hard_period_modes_converge_to_pqvi}
\label{lem:hard_period_spectral_radii_converge}
If \(\rho(I-\alpha\Phi^\top D\Phi)<1\), then for every deterministic policy
\(\pi\),
\[
\lim_{m\to\infty}A_{\pi,m}^{\mathrm{DLQL}}=A_\pi^{\mathrm{PQVI}}.
\]
Moreover, the switching families satisfy the JSR limit
\[
\lim_{m\to\infty}\rho\bigl(\mathcal A_m^{\mathrm{DLQL}}\bigr)
=
\rho\bigl(\mathcal A^{\mathrm{PQVI}}\bigr).
\]
\end{lemma}

The remaining step-size condition guarantees the residual decay used in the
preceding endpoint lemma and will be used throughout the \(\lambda\)-averaged
analysis. We present it separately because it is the only relaxation step-size
condition needed for the hard-target period maps.
\setcounter{lemma}{4}
\begin{lemma}
\label{lem:relaxation_stepsize_condition}
The matrix \(I-\alpha\Phi^\top D\Phi\) satisfies
\begin{align}
\rho(I-\alpha\Phi^\top D\Phi)<1
\quad\Longleftrightarrow\quad
0<\alpha<\frac{2}{\lambda_{\max}(\Phi^\top D\Phi)}.
\label{eq:relaxation_stepsize_condition}
\end{align}
\end{lemma}
We use this range as the standing step-size condition for the hard-target period
maps and the period-averaged analysis below. The following assumption keeps this
condition available without repeating the eigenvalue bound in every statement.
\begin{assumption}
\label{ass:hard_target_stepsize}
The step-size satisfies~\Cref{eq:relaxation_stepsize_condition}, namely \(0<\alpha<\frac{2}{\lambda_{\max}(\Phi^\top D\Phi)}\).
\end{assumption}

\section{\texorpdfstring{$\lambda$}{lambda}-DLQL: Geometrically Averaged Hard-Target Periods}
\label{sec:geometrically_averaged_period_targets}

The hard-target period maps above are indexed by one integer target-copy period
$m\in\{1,2,\ldots\}$. A continuous alternative is to average the candidate
targets obtained at the next target boundary. This use of a $\lambda$ parameter is
similar in spirit to TD($\lambda$) and Q($\lambda$): those methods use eligibility
traces or $\lambda$-weighted multi-step targets to interpolate between one-step
and longer-horizon updates~\citep{sutton1988learning,watkins1989learning,pengwilliams1996incremental,sutton1998reinforcement}.
We call the resulting TD($\lambda$)-style target-period extension
\emph{$\lambda$-DLQL}. This is not an $n$-step reward return and does not use
eligibility traces; the geometric weighting is applied to hard-target boundary
maps. It is a $\lambda$-discounted average over the $m$-DLQL matrices
$A_{\pi,1}^{\mathrm{DLQL}},A_{\pi,2}^{\mathrm{DLQL}},\ldots$ that map one target-boundary error to
the next. This continuous period parameterization is summarized in~\Cref{fig:period_lambda_interpolation}.
\begin{figure}[H]
\centering
\begin{tikzpicture}[>=Stealth, line cap=round, line join=round]
  \draw[line width=1.35pt] (0,0) -- (10,0);
  \fill[red] (0,0) circle (2.8pt);
  \fill (5,0) circle (2.8pt);
  \fill[blue] (10,0) circle (2.8pt);

  \node[align=center, anchor=south west] at (-0.75,0.35)
    {DLQL};
  \node[align=center, anchor=south east] at (10.75,0.35)
    {PQVI};

  \draw[thick, ->] (5,1.18) -- (5,0.13);
  \node[align=center, anchor=south] at (5,1.18) {$\lambda$-DLQL};

  \node[anchor=north] at (0,-0.30) {$\lambda=0$};
  \node[anchor=north] at (5,-0.30) {$\lambda$};
  \node[anchor=north] at (10,-0.30) {$\lambda=1$};

  \draw[thick, ->] (3.65,-1.08) -- (6.35,-1.08);
  \node[anchor=north] at (5,-1.20) {$\lambda\to1$};
\end{tikzpicture}
\caption{$\lambda$-DLQL target parameterization of the hard-target endpoints.
The parameter $\lambda=0$ recovers the period-one DLQL boundary update, while the
continuous endpoint $\lambda\to1$ corresponds to the infinite-period PQVI limit.}
\label{fig:period_lambda_interpolation}
\end{figure}
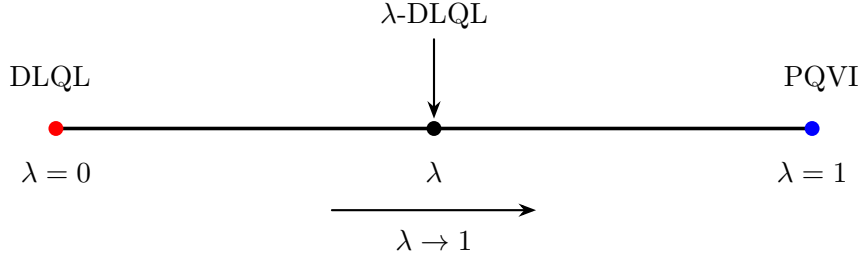
The interpolation in~\Cref{fig:period_lambda_interpolation} uses $\lambda$ to index the averaged period target: increasing $\lambda$ shifts the average toward longer hard-target periods, and the continuous endpoint is the PQVI limit.
At target boundary $n$, fix the boundary target $\bar\theta_n$. During one
boundary update, also fix the deterministic greedy mode $\pi_n$ selected by the
frozen target, so that
\[
\pi_n(s)\in\arg\max_{a\in\mathcal A}\phi(s,a)^\top\bar\theta_n,
\qquad s\in\mathcal S,
\]
with fixed deterministic tie breaking. Define the frozen-target one-step online
map
\begin{equation*}
h_{\pi_n}(\theta;\bar\theta_n)
:=
\theta+
\alpha\Phi^\top D
\left(R+\gamma P\Pi^{\pi_n}\Phi\bar\theta_n-\Phi\theta\right),
\end{equation*}
where $\pi_n$ is the deterministic mode selected by
$V_{\bar\theta_n}$ on the sampled next states. Equivalently, it can be written as
\begin{equation*}
h_{\pi_n}(\theta;\bar\theta_n)
=
(I-\alpha\Phi^\top D\Phi)\theta+
\alpha\left(\Phi^\top D R+\gamma\Phi^\top DP\Pi^{\pi_n}\Phi\bar\theta_n\right).
\end{equation*}
For this fixed mode, the period-$m$ candidate boundary target is obtained by
applying the frozen-target map $m$ times:
\begin{align}
h_{\pi_n}^{m}(\bar\theta_n;\bar\theta_n).
\label{eq:period_m_as_frozen_iterate}
\end{align}
The corresponding Bellman operator maps the boundary target $\bar\theta_n$ to the
geometrically averaged next boundary target
\begin{align}
\bar\theta_{n+1}
:=(1-\lambda)\sum_{m=1}^{\infty}\lambda^{m-1}
h_{\pi_n}^{m}(\bar\theta_n;\bar\theta_n).
\label{eq:lambda_averaged_target_parameter}
\end{align}
The definition in~\Cref{eq:lambda_averaged_target_parameter} is a conceptual
value-iteration description of the operator, not a practical literal
implementation. Directly carrying it out would require simulating infinitely
many candidate periods and averaging their boundary outputs, so the closed forms
and recursive solver below are the implementable representations.
The weights in~\Cref{eq:lambda_averaged_target_parameter} are the probabilities
of a geometric random variable on $\{1,2,\ldots\}$:
\[
\mathbb P(M=m)=(1-\lambda)\lambda^{m-1},
\qquad
\mathbb E[M]=\frac{1}{1-\lambda}.
\]
Thus $\lambda=0$ puts all mass on period one, while
$\lambda\uparrow1$ sends the effective period to infinity. The endpoint
$\lambda=1$ is understood as a continuous $\lambda\uparrow1$ limit, not by direct substitution into the
infinite weighted sum.
The next lemma turns this weighted period construction into the corresponding
switching error model.
\begin{lemma}
\label{lem:period_lambda_switched_boundary_model}
The period-$\lambda$ boundary errors are represented by the switching system
\[
\bar\theta_{n+1}-\theta^\star
=A_{\mu_n,\lambda}^{\mathrm{DLQL}}(\bar\theta_n-\theta^\star),
\]
where $\mu_n=\mu_{\bar\theta_n,\theta^\star}$ is a stochastic policy satisfying
\[
V_{\bar\theta_n}-V_{\theta^\star}
=\Pi^{\mu_n}\Phi(\bar\theta_n-\theta^\star).
\]
Here $\bar\theta_n$ denotes the boundary target at boundary $n$, namely
the target parameter used as the frozen target for that boundary update. The map
in this recursion is defined, for a stochastic policy $\mu$ with
$\mu(s)\in\Delta_{|\mathcal A|}$, by
\begin{align}
A_{\mu,\lambda}^{\mathrm{DLQL}}
:=(1-\lambda)
\sum_{m=1}^{\infty}\lambda^{m-1}A_{\mu,m}^{\mathrm{DLQL}},
\qquad 0\leq\lambda<1.
\label{eq:lambda_dlql_A_pi_lambda_definition}
\end{align}
For deterministic policies, we write $A_{\pi,\lambda}^{\mathrm{DLQL}}$. The associated
finite switched family is
\begin{align}
\mathcal A_\lambda^{\mathrm{DLQL}}
:=\set{A_{\pi,\lambda}^{\mathrm{DLQL}}:\pi\in\Theta}.
\label{eq:lambda_dlql_family_definition}
\end{align}
\end{lemma}

\begin{proof}
For each period length \(m\),~\Cref{lem:hard_period_boundary_map} gives the
period-\(m\) boundary error candidate
\(A_{\mu_n,m}^{\mathrm{DLQL}}(\bar\theta_n-\theta^\star)\), where the stochastic
policy \(\mu_n=\mu_{\bar\theta_n,\theta^\star}\) is supplied by~\Cref{lem:stochastic_policy_linearization}. Averaging these candidate boundary
errors with the same geometric weights used in~\Cref{eq:lambda_averaged_target_parameter} gives
\[
(1-\lambda)\sum_{m=1}^{\infty}\lambda^{m-1}A_{\mu_n,m}^{\mathrm{DLQL}}(\bar\theta_n-\theta^\star)
=A_{\mu_n,\lambda}^{\mathrm{DLQL}}(\bar\theta_n-\theta^\star).
\]
This is the displayed boundary-error recursion. For deterministic policies, the
finite policy set \(\Theta\) gives the associated switched family in~\Cref{eq:lambda_dlql_family_definition}.
\end{proof}

The following elementary commutation fact will be used to rearrange the
resolvent factors.
\begin{lemma}
\label{lem:target_factor_commutation}
For every \(0\leq\lambda\leq1\), the matrices
\(I-\alpha\Phi^\top D\Phi\) and
\(I-\lambda(I-\alpha\Phi^\top D\Phi)\) commute. If
\(I-\lambda(I-\alpha\Phi^\top D\Phi)\) is invertible, then
\(I-\alpha\Phi^\top D\Phi\) also commutes with
\(\left[I-\lambda(I-\alpha\Phi^\top D\Phi)\right]^{-1}\).
\end{lemma}
\begin{proof}
The identity
\[
(I-\alpha\Phi^\top D\Phi)\bigl[I-\lambda(I-\alpha\Phi^\top D\Phi)\bigr]
=
\bigl[I-\lambda(I-\alpha\Phi^\top D\Phi)\bigr](I-\alpha\Phi^\top D\Phi)
\]
follows by expanding both sides. If the second factor is invertible, multiplying
this identity on the left and on the right by its inverse gives the inverse
commutation relation.
\end{proof}

The endpoint proofs below use the following closed form, obtained by summing the
geometric matrix series in~\Cref{eq:lambda_dlql_A_pi_lambda_definition}.
\begin{lemma}
\label{lem:geometric_averaged_closed_form}
\label{lem:period_lambda_correction_matrix_invertible}
Under~\Cref{ass:hard_target_stepsize}, for \(0\leq\lambda<1\) and any stochastic
policy \(\mu\),
\begin{align}
A_{\mu,\lambda}^{\mathrm{DLQL}}
={}&(1-\lambda)(I-\alpha\Phi^\top D\Phi)
\left[I-\lambda(I-\alpha\Phi^\top D\Phi)\right]^{-1} \nonumber\\
&+\left(I-(1-\lambda)(I-\alpha\Phi^\top D\Phi)
\left[I-\lambda(I-\alpha\Phi^\top D\Phi)\right]^{-1}\right)A_\mu^{\mathrm{PQVI}}.
\label{eq:lambda_dlql_closed_form_A_pqvi}
\end{align}
Equivalently,
\begin{align}
A_{\mu,\lambda}^{\mathrm{DLQL}}
={}&(1-\lambda)(I-\alpha\Phi^\top D\Phi)
\left[I-\lambda(I-\alpha\Phi^\top D\Phi)\right]^{-1} \nonumber\\
&+\left[I-\lambda(I-\alpha\Phi^\top D\Phi)\right]^{-1}
\alpha\gamma\Phi^\top DP\Pi^\mu\Phi.
\label{eq:lambda_dlql_closed_form_B}
\end{align}
Moreover, for every \(0\leq\lambda\leq1\), the matrix
\[
I-\lambda(I-\alpha\Phi^\top D\Phi)=(1-\lambda)I+\lambda\alpha\Phi^\top D\Phi
\]
is symmetric positive definite, and in particular invertible.
\end{lemma}
\begin{proof}
First consider the matrix in the final statement. For any nonzero vector \(x\),
\[
x^\top\bigl[(1-\lambda)I+\lambda\alpha\Phi^\top D\Phi\bigr]x
=
(1-\lambda)\|x\|_2^2+
\lambda\alpha x^\top\Phi^\top D\Phi x.
\]
By~\Cref{ass:feature_sampling}, \(\Phi^\top D\Phi\) is symmetric positive definite.
If \(0\leq\lambda<1\), the first term is strictly positive; if \(\lambda=1\),
the second term is strictly positive because \(\alpha>0\). Hence the displayed
matrix is symmetric positive definite for every \(0\leq\lambda\leq1\), and is
therefore invertible.

By~\Cref{lem:period_map_decomposition}, for a fixed stochastic policy \(\mu\),
\[
A_{\mu,m}^{\mathrm{DLQL}}=A_\mu^{\mathrm{PQVI}}+(I-\alpha\Phi^\top D\Phi)^m(I-A_\mu^{\mathrm{PQVI}}).
\]
Therefore
\[
\begin{aligned}
A_{\mu,\lambda}^{\mathrm{DLQL}}
&=(1-\lambda)\sum_{m=1}^{\infty}\lambda^{m-1}A_{\mu,m}^{\mathrm{DLQL}} \\
&=(1-\lambda)\sum_{m=1}^{\infty}\lambda^{m-1}
(I-\alpha\Phi^\top D\Phi)^m \\
&\quad+
\left[I-(1-\lambda)\sum_{m=1}^{\infty}\lambda^{m-1}
(I-\alpha\Phi^\top D\Phi)^m\right]A_\mu^{\mathrm{PQVI}}.
\end{aligned}
\]
Applying~\Cref{lem:geometric_resolvent_identity} with
\(A=I-\alpha\Phi^\top D\Phi\) gives
\[
(1-\lambda)\sum_{m=1}^{\infty}\lambda^{m-1}
(I-\alpha\Phi^\top D\Phi)^m
=(1-\lambda)(I-\alpha\Phi^\top D\Phi)
\left[I-\lambda(I-\alpha\Phi^\top D\Phi)\right]^{-1},
\]
which proves~\Cref{eq:lambda_dlql_closed_form_A_pqvi}.

It remains to simplify the coefficient multiplying \(A_\mu^{\mathrm{PQVI}}\). Since
\(I-\lambda(I-\alpha\Phi^\top D\Phi)\) and \(I-\alpha\Phi^\top D\Phi\) commute,
we have
\[
\begin{aligned}
&I-(1-\lambda)(I-\alpha\Phi^\top D\Phi)
\left[I-\lambda(I-\alpha\Phi^\top D\Phi)\right]^{-1} \\
&\quad=
\left[I-\lambda(I-\alpha\Phi^\top D\Phi)\right]
\left[I-\lambda(I-\alpha\Phi^\top D\Phi)\right]^{-1} \\
&\qquad -(1-\lambda)(I-\alpha\Phi^\top D\Phi)
\left[I-\lambda(I-\alpha\Phi^\top D\Phi)\right]^{-1} \\
&\quad=
\left(I-\lambda(I-\alpha\Phi^\top D\Phi)
-(1-\lambda)(I-\alpha\Phi^\top D\Phi)\right)
\left[I-\lambda(I-\alpha\Phi^\top D\Phi)\right]^{-1} \\
&\quad=
\left(I-(I-\alpha\Phi^\top D\Phi)\right)
\left[I-\lambda(I-\alpha\Phi^\top D\Phi)\right]^{-1} \\
&\quad=
\alpha\Phi^\top D\Phi
\left[I-\lambda(I-\alpha\Phi^\top D\Phi)\right]^{-1} \\
&\quad=
\left[I-\lambda(I-\alpha\Phi^\top D\Phi)\right]^{-1}
\alpha\Phi^\top D\Phi.
\end{aligned}
\]
Thus the coefficient of \(A_\mu^{\mathrm{PQVI}}\) in~\Cref{eq:lambda_dlql_closed_form_A_pqvi} can also be written as
\[
I-(1-\lambda)(I-\alpha\Phi^\top D\Phi)
\left[I-\lambda(I-\alpha\Phi^\top D\Phi)\right]^{-1}
=
\left[I-\lambda(I-\alpha\Phi^\top D\Phi)\right]^{-1}
\alpha\Phi^\top D\Phi,
\]
where the last equality follows because
\(\alpha\Phi^\top D\Phi=I-(I-\alpha\Phi^\top D\Phi)\), and, under~\Cref{ass:hard_target_stepsize},
\[
\left[I-\lambda(I-\alpha\Phi^\top D\Phi)\right]^{-1}
=
\sum_{j=0}^{\infty}\lambda^j(I-\alpha\Phi^\top D\Phi)^j,
\qquad 0\leq\lambda<1.
\]
By~\Cref{lem:target_factor_commutation}, this inverse commutes with
\(I-\alpha\Phi^\top D\Phi\). Since
\(\alpha\Phi^\top D\Phi=I-(I-\alpha\Phi^\top D\Phi)\), it also
commutes with \(\alpha\Phi^\top D\Phi\).
Multiplying the displayed coefficient by \(A_\mu^{\mathrm{PQVI}}\) and using~\Cref{eq:pqvi_A_mu_definition},
\[
\alpha\Phi^\top D\Phi A_\mu^{\mathrm{PQVI}}
=
\alpha\Phi^\top D\Phi
(\Phi^\top D\Phi)^{-1}\gamma\Phi^\top DP\Pi^\mu\Phi
=
\alpha\gamma\Phi^\top DP\Pi^\mu\Phi,
\]
we obtain
\[
\left[I-\lambda(I-\alpha\Phi^\top D\Phi)\right]^{-1}
\alpha\Phi^\top D\Phi A_\mu^{\mathrm{PQVI}}
=
\left[I-\lambda(I-\alpha\Phi^\top D\Phi)\right]^{-1}
\alpha\gamma\Phi^\top DP\Pi^\mu\Phi.
\]
Substituting this identity into~\Cref{eq:lambda_dlql_closed_form_A_pqvi} gives~\Cref{eq:lambda_dlql_closed_form_B}.
\end{proof}

We call the map in~\Cref{eq:lambda_dlql_A_pi_lambda_definition} the \emph{$\lambda$-DLQL hard-target
map}. Instead of choosing one integer period $m$, the parameter $\lambda$ assigns
weights to all periods; moving $\lambda$ continuously shifts the weight from
short periods toward long periods. For the main comparison, the key consequences
are the endpoint identities and the induced JSR limits below.
\begin{lemma}
\label{lem:lambda_target_endpoint_maps}
\label{lem:lambda_target_jsr_endpoint_limits}
For every stochastic policy $\mu$,
\[
\mathop{\lim}\limits_{\lambda\downarrow0} A_{\mu,\lambda}^{\mathrm{DLQL}}
= A_\mu^{\mathrm{DLQL}},
\qquad
\mathop{\lim}\limits_{\lambda\uparrow1} A_{\mu,\lambda}^{\mathrm{DLQL}}
= A_\mu^{\mathrm{PQVI}}.
\]
Moreover,
\begin{align}
\lim_{\lambda\downarrow0}\rho(\mathcal A_\lambda^{\mathrm{DLQL}})
&=
\rho(\mathcal A^{\mathrm{DLQL}}),
\label{eq:lambda_dlql_zero_jsr_convergence}\\
\lim_{\lambda\uparrow1}\rho(\mathcal A_\lambda^{\mathrm{DLQL}})
&=
\rho(\mathcal A^{\mathrm{PQVI}}).
\label{eq:lambda_dlql_jsr_convergence}
\end{align}
\end{lemma}
\begin{proof}
Taking \(\lambda\downarrow0\) in~\Cref{eq:lambda_dlql_closed_form_B} gives
\[
\mathop{\lim}\limits_{\lambda\downarrow0} A_{\mu,\lambda}^{\mathrm{DLQL}}
=(I-\alpha\Phi^\top D\Phi)+\alpha\Phi^\top D\Phi A_\mu^{\mathrm{PQVI}}
=A_\mu^{\mathrm{DLQL}},
\]
where the last equality is~\Cref{eq:A_mu_definition}. By~\Cref{lem:period_map_decomposition}, with \(\mu\) in place of \(\pi\),
\(A_{\mu,m}^{\mathrm{DLQL}}\to A_\mu^{\mathrm{PQVI}}\) under~\Cref{ass:hard_target_stepsize}. Applying the Abel convergence lemma~\Cref{lem:abel_convergence_matrix_sequences} to the convergent matrix sequence
\((A_{\mu,m}^{\mathrm{DLQL}})_{m\geq1}\) gives
\[
(1-\lambda)\sum_{m\geq1}\lambda^{m-1}A_{\mu,m}^{\mathrm{DLQL}}
\to A_\mu^{\mathrm{PQVI}}
\qquad\text{as }\lambda\uparrow1.
\]
Thus \(A_{\mu,\lambda}^{\mathrm{DLQL}}\to A_\mu^{\mathrm{PQVI}}\).

For the JSR limits, apply the preceding endpoint convergences to every
deterministic policy \(\pi\in\Theta\). The set \(\Theta\) is finite, so the mode
matrices in \(\mathcal A_\lambda^{\mathrm{DLQL}}\) converge uniformly, mode by
mode, to the corresponding endpoint families. The JSR is continuous for finite
matrix families with respect to such perturbations~\citep{heilstrang1995continuity,jungers2009joint}. Therefore~\Cref{eq:lambda_dlql_zero_jsr_convergence} and~\Cref{eq:lambda_dlql_jsr_convergence} follow.
\end{proof}

The same standard Lyapunov argument now applies to the \(\lambda\)-averaged
boundary recursion. We record it as a separate lemma because it is the convergence
certificate used by the endpoint cases below.
\begin{lemma}
\label{lem:lambda_averaged_boundary_convergence_certificate}
Suppose either \(0\leq\lambda<1\) and \(\rho(\mathcal A_\lambda^{\mathrm{DLQL}})<1\),
or \(\lambda=1\) and \(\rho(\mathcal A^{\mathrm{PQVI}})<1\). Then the corresponding
boundary errors satisfy \(\bar\theta_k-\theta^\star\to0\). More precisely, in the
first case, for every \(\epsilon>0\) with
\(\rho(\mathcal A_\lambda^{\mathrm{DLQL}})+\epsilon<1\), the norm
\(p_\epsilon\), constant \(C_\epsilon\), and rate
\(\beta_\epsilon=\rho(\mathcal A_\lambda^{\mathrm{DLQL}})+\epsilon\) from~\Cref{lem:common_lyapunov_construction} give
\[
p_\epsilon(\bar\theta_k-\theta^\star)
\leq
\beta_\epsilon^k p_\epsilon(\bar\theta_0-\theta^\star),
\qquad
\|\bar\theta_k-\theta^\star\|_2
\leq
\sqrt{C_\epsilon}\,\beta_\epsilon^k\|\bar\theta_0-\theta^\star\|_2.
\]
In the endpoint case \(\lambda=1\), the same estimates hold with
\(\mathcal A^{\mathrm{PQVI}}\) in place of \(\mathcal A_\lambda^{\mathrm{DLQL}}\).
\end{lemma}
\begin{proof}
For \(0\leq\lambda<1\),~\Cref{lem:period_lambda_switched_boundary_model} gives
\[
\bar\theta_{k+1}-\theta^\star
=A_{\mu_k,\lambda}^{\mathrm{DLQL}}(\bar\theta_k-\theta^\star).
\]
Because \(\mu_k\) is a stochastic policy, the active matrix
\(A_{\mu_k,\lambda}^{\mathrm{DLQL}}\) lies in
\(\co(\mathcal A_\lambda^{\mathrm{DLQL}})\). Applying~\Cref{lem:common_lyapunov_construction} with
\(\mathcal H=\mathcal A_\lambda^{\mathrm{DLQL}}\) gives the displayed norm and
Euclidean estimates.

At \(\lambda=1\), the boundary map is the PQVI endpoint, whose active stochastic
mode is \(A_{\mu_k}^{\mathrm{PQVI}}\). This mode lies in
\(\co(\mathcal A^{\mathrm{PQVI}})\). Applying~\Cref{lem:common_lyapunov_construction} with
\(\mathcal H=\mathcal A^{\mathrm{PQVI}}\) gives the same estimates with the PQVI
family. In both cases the rate is strictly smaller than one, so the boundary
errors converge to zero.
\end{proof}

The endpoint identities above turn the qualitative interpolation picture into a
local JSR statement. The next theorem separates the four possible endpoint
configurations. It should be read as a neighborhood result: small $\lambda$
inherits the DLQL endpoint, while $\lambda$ close to one inherits the PQVI endpoint.
\begin{theorem}
\label{thm:lambda_target_endpoint_cases}
The following endpoint cases hold.
\begin{enumerate}[label=(\roman*)]
\item If $\rho(\mathcal A^{\mathrm{PQVI}})<1$ and $\rho(\mathcal A^{\mathrm{DLQL}})<1$, then
$\rho(\mathcal A_\lambda^{\mathrm{DLQL}})<1$ for all sufficiently small $\lambda\geq0$
and also for all $\lambda$ sufficiently close to one.
\item If $\rho(\mathcal A^{\mathrm{PQVI}})<1<\rho(\mathcal A^{\mathrm{DLQL}})$, then
$\rho(\mathcal A_\lambda^{\mathrm{DLQL}})>1$ for all sufficiently small $\lambda\geq0$,
whereas for all $\lambda$ sufficiently close to one, $\rho(\mathcal A_\lambda^{\mathrm{DLQL}})<1$.
\item If $\rho(\mathcal A^{\mathrm{DLQL}})<1<\rho(\mathcal A^{\mathrm{PQVI}})$, then the small-$\lambda$
endpoint is JSR-stable and the large-$\lambda$ endpoint fails the JSR stability
test: for all $\lambda$ sufficiently close to zero and all $\lambda'$ sufficiently
close to one, $\rho(\mathcal A_\lambda^{\mathrm{DLQL}})
<1<
\rho(\mathcal A_{\lambda'}^{\mathrm{DLQL}})$.
\item If $\rho(\mathcal A^{\mathrm{PQVI}})>1$ and $\rho(\mathcal A^{\mathrm{DLQL}})>1$, then
$\rho(\mathcal A_\lambda^{\mathrm{DLQL}})>1$ for all sufficiently small $\lambda\geq0$
and also for all $\lambda$ sufficiently close to one.
\end{enumerate}
\end{theorem}
\begin{proof}
By~\Cref{lem:lambda_target_jsr_endpoint_limits},
\[
\rho(\mathcal A_\lambda^{\mathrm{DLQL}})
\to \rho(\mathcal A^{\mathrm{DLQL}})
\quad\text{as }\lambda\downarrow0,
\qquad
\rho(\mathcal A_\lambda^{\mathrm{DLQL}})
\to \rho(\mathcal A^{\mathrm{PQVI}})
\quad\text{as }\lambda\uparrow1.
\]
These endpoint limits are obtained from finite-family JSR continuity
\citep{heilstrang1995continuity,jungers2009joint}. We now use the strict margins
around the threshold value one.

If an endpoint JSR value is strictly less than one, say it is \(r<1\), choose
\(\delta=(1-r)/2>0\). The corresponding endpoint limit gives a neighborhood of
that endpoint on which
\(\rho(\mathcal A_\lambda^{\mathrm{DLQL}})<r+\delta=(1+r)/2<1\). Thus JSR
stability is preserved near that endpoint. If an endpoint JSR value is strictly
greater than one, say it is \(r>1\), choose \(\delta=(r-1)/2>0\). The same
endpoint limit gives a neighborhood on which
\(\rho(\mathcal A_\lambda^{\mathrm{DLQL}})>r-
\delta=(1+r)/2>1\). Thus failure of the strict JSR stability test is also
preserved near that endpoint.

Applying the first alternative at each stable endpoint and the second alternative
at each unstable endpoint gives the four cases (i)--(iv).
\end{proof}
Thus small $\lambda$ is favorable when the DLQL family is stable and the
PQVI endpoint is unstable, while large $\lambda$ is favorable when the
PQVI family is stable or has the strict endpoint JSR advantage. If both
endpoints fail the JSR stability test, an endpoint argument cannot certify
stability, and only an intermediate choice of $\lambda$ can possibly improve the
JSR\@. As above, a JSR above one is a worst-case obstruction to this certificate,
not an automatic divergence statement for every realized switching path.

\section{Recursive Solver for \texorpdfstring{$\lambda$}{lambda}-DLQL}
\label{sec:recursive_period_lambda_solver}

The definition of $\lambda$-DLQL in~\Cref{eq:lambda_averaged_target_parameter}
is conceptually useful, but direct simulation of all infinitely many target
periods is not implementable. This section gives the recursive correction form
that computes the same target with one linear solve at each target boundary. The following algorithm,~\Cref{alg:exact_deterministic_model_based_period_lambda_target}, gives the deterministic
model-based boundary update corresponding to this correction formula.
\setcounter{algorithm}{3}
\begin{algorithm}[H]
\caption{Exact Deterministic Model-Based $\lambda$-DLQL Boundary Update}
\label{alg:exact_deterministic_model_based_period_lambda_target}
\begin{algorithmic}[1]
\REQUIRE Initial boundary target $\bar\theta_0$, step-size $\alpha>0$, and parameter $\lambda\in[0,1]$.
\FOR{target boundaries $n=0,1,2,\ldots$}
\STATE Select the deterministic mode $\pi_n$ by the greedy rule
$\displaystyle
\pi_n(s)\in\arg\max_{a\in\mathcal A}\phi(s,a)^\top\bar\theta_n
$
for every $s\in\mathcal S$, with fixed deterministic tie breaking; equivalently,
$V_{\bar\theta_n}=\Pi^{\pi_n}\Phi\bar\theta_n$.
\STATE Freeze $\bar\theta_n$.
\STATE Update the boundary target directly by
\begin{equation*}
\bar\theta_{n+1}
=
\bar\theta_n
+
\alpha\bigl[(1-\lambda)I+\lambda\alpha\Phi^\top D\Phi\bigr]^{-1}
\Phi^\top D
\left(
R+\gamma P\Pi^{\pi_n}\Phi\bar\theta_n-\Phi\bar\theta_n
\right).
\end{equation*}
\STATE For a hard boundary implementation, set the online parameter to
$\bar\theta_{n+1}$ before starting the next boundary block.
\ENDFOR
\end{algorithmic}
\end{algorithm}

Algorithm~\ref{alg:exact_deterministic_model_based_period_lambda_target} states
the computation performed at each target boundary. The next proposition justifies
it: the single correction equation, equivalently the direct inverse update used
in the algorithm, returns exactly the same geometrically averaged target
as~\Cref{eq:lambda_averaged_target_parameter}, without explicitly simulating all
candidate periods.

\begin{proposition}
\label{prop:recursive_lambda_target}
Let $0\leq\lambda<1$. Under~\Cref{ass:hard_target_stepsize}, fix a boundary
target $\bar\theta_n$ and the corresponding deterministic greedy mode $\pi_n$.
Define the next boundary target by the geometrically averaged hard-target map
\begin{equation*}
\bar\theta_{n+1}
:=(1-\lambda)
\sum_{m=1}^{\infty}\lambda^{m-1}
h_{\pi_n}^{m}(\bar\theta_n;\bar\theta_n).
\end{equation*}
Then the series converges, and $\bar\theta_{n+1}$ is the unique solution of
\begin{align}
\bar\theta_{n+1}
=
(1-\lambda)h_{\pi_n}(\bar\theta_n;\bar\theta_n)
+
\lambda h_{\pi_n}\bigl(\bar\theta_{n+1};\bar\theta_n\bigr).
\label{eq:lambda_parameter_recursive_equation}
\end{align}
Equivalently, the boundary correction $\bar\theta_{n+1}-\bar\theta_n$ satisfies
\begin{align}
\bar\theta_{n+1}-\bar\theta_n
=
h_{\pi_n}(\bar\theta_n;\bar\theta_n)-\bar\theta_n
+\lambda(I-\alpha\Phi^\top D\Phi)(\bar\theta_{n+1}-\bar\theta_n),
\label{eq:gae_style_period_recursion}
\end{align}
or, equivalently,
\begin{align}
\bigl[(1-\lambda)I+\lambda\alpha\Phi^\top D\Phi\bigr]
(\bar\theta_{n+1}-\bar\theta_n)
=
h_{\pi_n}(\bar\theta_n;\bar\theta_n)-\bar\theta_n.
\label{eq:period_lambda_linear_system}
\end{align}
Therefore,
\begin{equation*}
\bar\theta_{n+1}
=
\bar\theta_n
+
\alpha\bigl[(1-\lambda)I+\lambda\alpha\Phi^\top D\Phi\bigr]^{-1}
\Phi^\top D
\left(R+\gamma P\Pi^{\pi_n}\Phi\bar\theta_n-\Phi\bar\theta_n\right).
\end{equation*}
\end{proposition}
\begin{proof}
We first justify convergence of the weighted series defining \(\bar\theta_{n+1}\).
For the fixed boundary target \(\bar\theta_n\) and fixed mode \(\pi_n\), repeated
application of the affine frozen-target map gives
\[
\begin{aligned}
h_{\pi_n}^{m}(\bar\theta_n;\bar\theta_n)
={}&(I-\alpha\Phi^\top D\Phi)^m\bar\theta_n \\
&+\sum_{i=0}^{m-1}(I-\alpha\Phi^\top D\Phi)^i
\alpha\left(\Phi^\top D R+
\gamma\Phi^\top DP\Pi^{\pi_n}\Phi\bar\theta_n\right).
\end{aligned}
\]
By~\Cref{ass:hard_target_stepsize}, \(\rho(I-\alpha\Phi^\top D\Phi)<1\). Hence
there exist constants \(C\geq1\) and \(r\in(0,1)\) such that
\[
\|(I-\alpha\Phi^\top D\Phi)^m\|_2\leq Cr^m,
\qquad m\geq0.
\]
The displayed iterate is therefore uniformly bounded in \(m\), because
\(\sum_{i=0}^{m-1}Cr^i\leq C/(1-r)\). Since
\(\sum_{m=1}^{\infty}(1-\lambda)\lambda^{m-1}=1\) for \(0\leq\lambda<1\), the
weighted series in~\Cref{eq:lambda_averaged_target_parameter} converges
absolutely.

Using the $m$-fold frozen-target iterate in~\Cref{eq:period_m_as_frozen_iterate} and shifting the summation index,
\begin{align}
\bar\theta_{n+1}
&=(1-\lambda)h_{\pi_n}(\bar\theta_n;\bar\theta_n)
+(1-\lambda)\sum_{m=2}^{\infty}\lambda^{m-1}
h_{\pi_n}^{m}(\bar\theta_n;\bar\theta_n)\notag\\
&=(1-\lambda)h_{\pi_n}(\bar\theta_n;\bar\theta_n)
+\lambda(1-\lambda)
\sum_{k=1}^{\infty}\lambda^{k-1}
h_{\pi_n}\left(h_{\pi_n}^{k}(\bar\theta_n;\bar\theta_n);\bar\theta_n\right).
\label{eq:lambda_shifted_sum_identity}
\end{align}
The map $\theta\mapsto h_{\pi_n}(\theta;\bar\theta_n)$ is affine. Since the
weights $(1-\lambda)\lambda^{k-1}$ sum to one, the affine map can be passed
through this weighted average:
\[
\begin{aligned}
&(1-\lambda)
\sum_{k=1}^{\infty}\lambda^{k-1}
 h_{\pi_n}\!\left(h_{\pi_n}^{k}(\bar\theta_n;\bar\theta_n);\bar\theta_n\right) \\
&\quad=
h_{\pi_n}\!\left(
(1-\lambda)\sum_{k=1}^{\infty}\lambda^{k-1}
 h_{\pi_n}^{k}(\bar\theta_n;\bar\theta_n);
\bar\theta_n
\right).
\end{aligned}
\]
The argument of $h_{\pi_n}$ on the right is exactly $\bar\theta_{n+1}$ by~\Cref{eq:lambda_averaged_target_parameter}. Substituting this identity into~\Cref{eq:lambda_shifted_sum_identity}
gives~\Cref{eq:lambda_parameter_recursive_equation}.
Now write the next target as $\bar\theta_{n+1}=\bar\theta_n+(\bar\theta_{n+1}-\bar\theta_n)$.
From the affine representation
\[
h_{\pi_n}(\theta;\bar\theta_n)
=
(I-\alpha\Phi^\top D\Phi)\theta+
\alpha\left(\Phi^\top D R+
\gamma\Phi^\top DP\Pi^{\pi_n}\Phi\bar\theta_n\right),
\]
the difference between two evaluations with the same frozen target satisfies
\[
h_{\pi_n}(\theta;\bar\theta_n)-h_{\pi_n}(\bar\theta_n;\bar\theta_n)
=
(I-\alpha\Phi^\top D\Phi)(\theta-\bar\theta_n).
\]
Taking $\theta=\bar\theta_{n+1}$ gives
\[
h_{\pi_n}(\bar\theta_{n+1};\bar\theta_n)
=
h_{\pi_n}(\bar\theta_n;\bar\theta_n)
+(I-\alpha\Phi^\top D\Phi)(\bar\theta_{n+1}-\bar\theta_n).
\]
Equivalently, after adding and subtracting $\bar\theta_n$ inside the first term,
\[
h_{\pi_n}(\bar\theta_{n+1};\bar\theta_n)
=
\bar\theta_n+
\bigl(h_{\pi_n}(\bar\theta_n;\bar\theta_n)-\bar\theta_n\bigr)
+(I-\alpha\Phi^\top D\Phi)(\bar\theta_{n+1}-\bar\theta_n).
\]
Substituting this identity into~\Cref{eq:lambda_parameter_recursive_equation}
gives~\Cref{eq:gae_style_period_recursion}. Finally,
$I-\lambda(I-\alpha\Phi^\top D\Phi)=(1-\lambda)I+
\lambda\alpha\Phi^\top D\Phi$, which yields~\Cref{eq:period_lambda_linear_system}. For $0\leq\lambda<1$, the matrix
$(1-\lambda)I+\lambda\alpha\Phi^\top D\Phi$ is symmetric positive definite, so
the solution is unique. Solving the linear system and using the expression for
$h_{\pi_n}(\bar\theta_n;\bar\theta_n)-\bar\theta_n$ gives the displayed direct inverse
formula for $\bar\theta_{n+1}$.
\end{proof}

The right-hand side of the correction equation in~\Cref{eq:period_lambda_linear_system} is the period-one frozen-target correction,
namely
\[
h_{\pi_n}(\bar\theta_n;\bar\theta_n)-\bar\theta_n
=
\alpha\Phi^\top D
\left(R+\gamma P\Pi^{\pi_n}\Phi\bar\theta_n-\Phi\bar\theta_n\right).
\]
At $\lambda=0$,
\begin{equation*}
\bar\theta_{n+1}-\bar\theta_n
=h_{\pi_n}(\bar\theta_n;\bar\theta_n)-\bar\theta_n,
\qquad
\bar\theta_{n+1}=h_{\pi_n}(\bar\theta_n;\bar\theta_n),
\end{equation*}
which is the period-one DLQL boundary update. At the endpoint $\lambda=1$, the
same correction equation becomes
\begin{equation*}
\alpha\Phi^\top D\Phi(\bar\theta_{n+1}-\bar\theta_n)
=h_{\pi_n}(\bar\theta_n;\bar\theta_n)-\bar\theta_n,
\end{equation*}
and hence
\begin{equation*}
\bar\theta_{n+1}
=
(\Phi^\top D\Phi)^{-1}\Phi^\top D
\left(R+\gamma P\Pi^{\pi_n}\Phi\bar\theta_n\right),
\end{equation*}
which is the PQVI boundary map for the frozen greedy mode. Therefore, for
$0\leq\lambda<1$, the recursive implementation in~\Cref{alg:exact_deterministic_model_based_period_lambda_target} computes the
same target as the geometric hard-period average in~\Cref{eq:lambda_averaged_target_parameter}; at $\lambda=1$, the same correction
equation gives the PQVI endpoint used in the switching-error statement below.
The correction equation in~\Cref{eq:period_lambda_linear_system} is the central
computational object shared by the inverse-form and inverse-free implementations.
Its matrix is explicit:
\[
(1-\lambda)I+\lambda\alpha\Phi^\top D\Phi
=I-\lambda(I-\alpha\Phi^\top D\Phi),
\]
which is the same matrix appearing in~\Cref{lem:geometric_averaged_closed_form}.
Since that lemma includes its positive definiteness and invertibility, the
correction equation is well posed for every $0\leq\lambda\leq1$; for
$0\leq\lambda<1$ its solution is the update already stated in~\Cref{prop:recursive_lambda_target}, and the endpoint $\lambda=1$ gives the PQVI
boundary map used below.

With the well-posed correction equation in place, we now present the corresponding
error recursion. The inverse formula does not create a different stability object;
after subtracting the projected Q-Bellman fixed point, the active stochastic modes
are the inverse-form modes $A_{\mu,\lambda}^{\mathrm{DLQL2}}$ below. The next lemma
then identifies these modes with the period-$\lambda$ modes from~\Cref{eq:lambda_dlql_A_pi_lambda_definition} for $0\leq\lambda<1$, and the endpoint
mode is the PQVI mode.
\setcounter{lemma}{1}
\begin{lemma}
\label{prop:recursive_lambda_inverse_error_switching}
Fix $0\leq\lambda\leq1$. Let $\bar\theta_n$ be generated by the boundary update
obtained from~\Cref{eq:period_lambda_linear_system}. At each boundary, choose a
deterministic greedy policy $\pi_n$ satisfying
\[
\pi_n(s)\in\arg\max_{a\in\mathcal A}\phi(s,a)^\top\bar\theta_n,
\qquad s\in\mathcal S,
\]
with fixed deterministic tie breaking, so that
$V_{\bar\theta_n}=\Pi^{\pi_n}\Phi\bar\theta_n$. Then there exists a stochastic
policy $\mu_n=\mu_{\bar\theta_n,\theta^\star}$, supplied by~\Cref{lem:stochastic_policy_linearization}, such that, for $0\leq\lambda<1$,
\[
\bar\theta_{n+1}-\theta^\star
=
A_{\mu_n,\lambda}^{\mathrm{DLQL2}}(\bar\theta_n-\theta^\star),
\]
where, for every stochastic policy $\mu$ and $0\leq\lambda\leq1$,
\[
A_{\mu,\lambda}^{\mathrm{DLQL2}}
:=
I+\bigl[(1-\lambda)I+\lambda\alpha\Phi^\top D\Phi\bigr]^{-1}
\alpha\Phi^\top D(\gamma P\Pi^\mu\Phi-\Phi).
\]
For deterministic policies, $A_{\pi,\lambda}^{\mathrm{DLQL2}}$ is defined by the
same formula with $\Pi^\pi$ in place of $\Pi^\mu$. At the endpoint $\lambda=1$,
this definition gives $A_{\mu,1}^{\mathrm{DLQL2}}=A_\mu^{\mathrm{PQVI}}$, and hence
\[
\bar\theta_{n+1}-\theta^\star
=
A_{\mu_n}^{\mathrm{PQVI}}(\bar\theta_n-\theta^\star).
\]
\end{lemma}
\begin{proof}
The projected Q-Bellman fixed point satisfies
\[
\Phi^\top D\left(R+\gamma P V_{\theta^\star}-\Phi\theta^\star\right)=0.
\]
By~\Cref{lem:stochastic_policy_linearization}, applied to $\bar\theta_n$ and
$\theta^\star$, there exists a stochastic policy
$\mu_n=\mu_{\bar\theta_n,\theta^\star}$ such that
\[
V_{\bar\theta_n}-V_{\theta^\star}
=
\Pi^{\mu_n}\Phi(\bar\theta_n-\theta^\star).
\]
Subtracting the fixed-point residual from the inverse-form boundary update gives
\[
\begin{aligned}
\bar\theta_{n+1}-\theta^\star
={}&
\left[
I+
\bigl((1-\lambda)I+\lambda\alpha\Phi^\top D\Phi\bigr)^{-1}
\alpha\Phi^\top D\bigl(\gamma P\Pi^{\mu_n}\Phi-\Phi\bigr)
\right]
(\bar\theta_n-\theta^\star).
\end{aligned}
\]
For $0\leq\lambda<1$, the bracketed matrix is
$A_{\mu_n,\lambda}^{\mathrm{DLQL2}}$ by definition. At $\lambda=1$, the same
bracketed matrix is
\[
I+
(\alpha\Phi^\top D\Phi)^{-1}
\alpha\Phi^\top D(\gamma P\Pi^{\mu_n}\Phi-\Phi)
=
(\Phi^\top D\Phi)^{-1}\gamma\Phi^\top DP\Pi^{\mu_n}\Phi
=A_{\mu_n}^{\mathrm{PQVI}},
\]
using~\Cref{eq:pqvi_A_mu_definition}. This proves both displayed recursions.
\end{proof}

\begin{lemma}
\label{lem:lambda_dlql2_equals_lambda_dlql}
Under~\Cref{ass:hard_target_stepsize}, for every stochastic policy $\mu$ and
$0\leq\lambda<1$,
\[
A_{\mu,\lambda}^{\mathrm{DLQL2}}
=
A_{\mu,\lambda}^{\mathrm{DLQL}}.
\]
The same identity holds for deterministic policies by taking $\mu=\pi$.
\end{lemma}
\begin{proof}
Using
\[
(1-\lambda)I+\lambda\alpha\Phi^\top D\Phi
=I-\lambda(I-\alpha\Phi^\top D\Phi),
\]
and~\Cref{lem:target_factor_commutation}, we have
\[
\begin{aligned}
A_{\mu,\lambda}^{\mathrm{DLQL2}}
&=I+
\bigl[I-\lambda(I-\alpha\Phi^\top D\Phi)\bigr]^{-1}
\alpha\Phi^\top D(\gamma P\Pi^\mu\Phi-\Phi) \\
&=\bigl[I-\lambda(I-\alpha\Phi^\top D\Phi)\bigr]^{-1}
\left[(1-\lambda)(I-\alpha\Phi^\top D\Phi)
+\alpha\gamma\Phi^\top DP\Pi^\mu\Phi\right] \\
&=(1-\lambda)(I-\alpha\Phi^\top D\Phi)
\bigl[I-\lambda(I-\alpha\Phi^\top D\Phi)\bigr]^{-1} \\
&\quad+
\bigl[I-\lambda(I-\alpha\Phi^\top D\Phi)\bigr]^{-1}
\alpha\gamma\Phi^\top DP\Pi^\mu\Phi.
\end{aligned}
\]
The last display is exactly the closed form in~\Cref{eq:lambda_dlql_closed_form_B}. Hence
$A_{\mu,\lambda}^{\mathrm{DLQL2}}=A_{\mu,\lambda}^{\mathrm{DLQL}}$.
\end{proof}

The next remarks clarify how the inverse-form update relates to the original
projected equation, to regularized PQVI, and to the inverse-free
implementation introduced next.
\begin{remark}
If the boundary sequence converges, say $\bar\theta_n\to\bar\theta_\infty$, then
$\bar\theta_{n+1}-\bar\theta_n\to 0$. Since~\Cref{lem:period_lambda_correction_matrix_invertible} gives invertibility and $V_\theta$ is continuous in $\theta$, the limit must satisfy
\[
\Phi^\top D\left(R+\gamma P V_{\bar\theta_\infty}-\Phi\bar\theta_\infty\right)=0,
\]
which is exactly the original projected Q-Bellman equation in~\Cref{eq:projected_q_bellman_equation}.
\end{remark}
Having identified the limiting equation, we next contrast it with regularized
PQVI, whose fixed point solves a modified residual equation.
\begin{remark}
The update resembles the regularized PQVI developed in~\citep{limlee2024regularized},
\[
\theta_{k+1}
=
\bigl[\eta I+\Phi^\top D\Phi\bigr]^{-1}\Phi^\top D
\left(R+\gamma P\Pi^{\pi_k}\Phi\theta_k-\Phi\theta_k\right),
\]
where $\theta_k$ is the online parameter at time $k$, $\eta>0$ is the regularization weight, and $\pi_k$ is the tie-broken greedy policy corresponding to $\theta_k$.
With regularization weight $\eta>0$, the fixed point of the regularized PQVI satisfies the modified residual equation
\[
\Phi^\top D\left(R+\gamma P V_\theta-\Phi\theta\right)-\eta\theta=0.
\]
Thus its limit generally solves a different, biased equation unless $\eta=0$. This distinction is important for well-posedness. The $\lambda$-DLQL recursion remains tied to the original projected Q-Bellman equation; if that equation has multiple solutions, the recursion need not select a unique limit, and if the equation has no solution, there is no projected Q-Bellman fixed point for the limit statement above. By contrast, regularized PQVI changes the equation by the term $-\eta\theta$. For sufficiently large $\eta$, the regularized map becomes a contraction and therefore has a unique fixed point, although that fixed point is generally biased relative to the unregularized projected equation.
\end{remark}

The dimension of \((1-\lambda)I+\lambda\alpha\Phi^\top D\Phi\) depends only on the
feature dimension. Hence it is often much smaller than the state-action space
and is computable in many linear-approximation settings. If this inverse is
still burdensome, the inverse-free recursion in the next section gives an
auxiliary alternative.

\section{Inverse-Free Recursion}
\label{sec:gtd_style_two_line_inverse_free_recursion}

The inverse-form solver in~\Cref{sec:recursive_period_lambda_solver} computes the
period-$\lambda$ correction by solving~\Cref{eq:period_lambda_linear_system}
exactly at each target boundary. When applying the explicit inverse is
undesirable, the same correction equation can instead be tracked by an auxiliary
recursion. This section first interprets the auxiliary variable as a Richardson
iteration~\citep[Section~3.1]{varga2000matrix} for that correction equation, states the $N$-step inner Richardson
version, and then gives the $N=1$ two-line implementation in~\Cref{sec:gtd_style_one_step_two_line_case}. For a linear system $Mz=b$, the
Richardson iteration is
\[
z^{(\ell+1)}=z^{(\ell)}+\eta\bigl(b-Mz^{(\ell)}\bigr),
\]
and, if $z^\star$ solves $Mz=b$, the error satisfies
\[
z^{(\ell+1)}-z^\star=(I-\eta M)(z^{(\ell)}-z^\star).
\]
The following is the standard Richardson convergence condition for symmetric
positive definite systems~\citep[Section~4.1]{saad2003iterative} and \citep[Section~3.1]{varga2000matrix}.
\begin{lemma}
\label{lem:richardson_convergence_condition}
Suppose that the matrix $M$ in the Richardson iteration above is symmetric
positive definite. If
\[
0<\eta<\frac{2}{\lambda_{\max}(M)},
\]
then $\rho(I-\eta M)<1$. Consequently, for every initial vector, the Richardson
iterates converge linearly to the solution $z^\star$.
\end{lemma}
\begin{proof}
Every eigenvalue of $M$ is positive. If $q$ is an eigenvalue of $M$, then
$0<q\leq\lambda_{\max}(M)$, and the displayed step-size condition gives
$-1<1-\eta q<1$. Since the eigenvalues of $I-\eta M$ are exactly the numbers
$1-\eta q$, the spectral radius of $I-\eta M$ is strictly smaller than one. The
error recursion displayed above therefore converges to zero geometrically.
\end{proof}
At target boundary $n$, choose the greedy policy generated by the current target,
\begin{equation}
\pi_n(s)\in\arg\max_{a\in\mathcal A}\phi(s,a)^\top\bar\theta_n,
\qquad s\in\mathcal S,
\label{eq:gtd_style_greedy_policy}
\end{equation}
with the same fixed deterministic tie breaking as before, so that
$V_{\bar\theta_n}=\Pi^{\pi_n}\Phi\bar\theta_n$.

\subsection{Frozen \texorpdfstring{$N$}{N}-Step Inner Richardson Interpretation}
\label{sec:gtd_style_inner_richardson_interpretation}

To see what the auxiliary Richardson step computes, freeze the boundary target
$\bar\theta_n$ and the greedy mode $\pi_n$. With these two quantities fixed, the
correction used by the inverse-form update is the unique vector
$\bar\theta_{n+1}-\bar\theta_n$ solving
\[
\bigl[(1-\lambda)I+\lambda\alpha\Phi^\top D\Phi\bigr]
(\bar\theta_{n+1}-\bar\theta_n)
=
\alpha\Phi^\top D
\left(R+\gamma P\Pi^{\pi_n}\Phi\bar\theta_n-\Phi\bar\theta_n\right).
\]
The corresponding Richardson iteration~\citep[Section~3.1]{varga2000matrix} for
this fixed linear system is:
\[
w^{(\ell+1)}
=
w^{(\ell)}+
\eta\left[
\alpha\Phi^\top D
\left(R+\gamma P\Pi^{\pi_n}\Phi\bar\theta_n-\Phi\bar\theta_n\right)
-\bigl[(1-\lambda)I+\lambda\alpha\Phi^\top D\Phi\bigr]w^{(\ell)}
\right].
\]
Subtracting the solution $\bar\theta_{n+1}-\bar\theta_n$ from both sides gives the
tracking-error recursion
\begin{equation}
w^{(\ell+1)}-(\bar\theta_{n+1}-\bar\theta_n)
=
\left(I-\eta\bigl[(1-\lambda)I+\lambda\alpha\Phi^\top D\Phi\bigr]\right)
\left(w^{(\ell)}-(\bar\theta_{n+1}-\bar\theta_n)\right).
\label{eq:gtd_style_frozen_tracking_error}
\end{equation}
Thus the auxiliary variable is not solving a new Bellman equation; for frozen
$(\bar\theta_n,\pi_n)$ it is simply trying to solve the same correction equation
as the inverse-form recursion in~\Cref{eq:period_lambda_linear_system} and~\Cref{alg:exact_deterministic_model_based_period_lambda_target}. The auxiliary
Richardson step-size condition is recorded as an assumption for the inverse-free
constructions.

\begin{assumption}
\label{ass:gtd_style_inner_stepsize}
The auxiliary step-size satisfies
\begin{equation}
0<\eta<\frac{2}{\lambda_{\max}\bigl((1-\lambda)I+\lambda\alpha\Phi^\top D\Phi\bigr)}.
\label{eq:gtd_style_inner_stepsize_condition}
\end{equation}
\end{assumption}

\begin{lemma}
\label{lem:gtd_style_richardson_iteration_matrix_stability}
Under~\Cref{ass:gtd_style_inner_stepsize}, every eigenvalue of
\[
I-\eta\bigl[(1-\lambda)I+\lambda\alpha\Phi^\top D\Phi\bigr]
\]
lies in $(-1,1)$. Consequently, the iteration matrix in~\Cref{eq:gtd_style_frozen_tracking_error} has spectral radius strictly
smaller than one.
\end{lemma}
\begin{proof}
By~\Cref{lem:period_lambda_correction_matrix_invertible}, the matrix
$ (1-\lambda)I+\lambda\alpha\Phi^\top D\Phi $ is symmetric positive
definite. Applying~\Cref{lem:richardson_convergence_condition} with this matrix
and the step-size condition in~\Cref{ass:gtd_style_inner_stepsize} gives
\[
\rho\left(I-\eta\bigl[(1-\lambda)I+\lambda\alpha\Phi^\top D\Phi\bigr]\right)<1.
\]
The iteration matrix is symmetric, so all its eigenvalues are real. Since its
spectral radius is strictly smaller than one, every eigenvalue lies in $(-1,1)$.
\end{proof}

\begin{lemma}
\label{lem:gtd_style_frozen_richardson_convergence}
Under~\Cref{ass:gtd_style_inner_stepsize}, for fixed $(\bar\theta_n,\pi_n)$,
the Richardson iteration above is linearly convergent to the unique solution of~\Cref{eq:period_lambda_linear_system}, namely
\[
\bigl[(1-\lambda)I+\lambda\alpha\Phi^\top D\Phi\bigr]^{-1}\alpha\Phi^\top D
\left(R+\gamma P\Pi^{\pi_n}\Phi\bar\theta_n-\Phi\bar\theta_n\right).
\]
Equivalently, the correction $\bar\theta_{n+1}-\bar\theta_n$ is the unique
solution satisfying the zero-residual equation
\[
\bigl[(1-\lambda)I+\lambda\alpha\Phi^\top D\Phi\bigr]
(\bar\theta_{n+1}-\bar\theta_n)
-
\alpha\Phi^\top D
\left(R+\gamma P\Pi^{\pi_n}\Phi\bar\theta_n-\Phi\bar\theta_n\right)
=0.
\]
\end{lemma}
\begin{proof}
By~\Cref{lem:period_lambda_correction_matrix_invertible}, the correction
matrix is symmetric positive definite and hence the correction equation has a
unique solution. By~\Cref{lem:gtd_style_richardson_iteration_matrix_stability},
the iteration matrix in~\Cref{eq:gtd_style_frozen_tracking_error} has
spectral radius strictly smaller than one. The tracking error therefore
converges to zero geometrically, which proves the claimed linear convergence
to the unique zero-residual solution of the correction equation.
\end{proof}

By~\Cref{lem:gtd_style_frozen_richardson_convergence}, if $\bar\theta_n$ and
$\pi_n$ were held fixed while the inner iteration was run to convergence, then
$w^{(\ell)}$ would converge to the correction displayed in the lemma.
This fixed-inner view leads first to the following version, which runs $N$ inner
Richardson subiterations before moving the boundary target.

\setcounter{algorithm}{4}
\begin{algorithm}[H]
\caption{$N$-Step Inner Richardson Version of the Inverse-Free Update}
\label{alg:gtd_style_n_step_inner_richardson}
\begin{algorithmic}[1]
\REQUIRE Initial boundary target $\bar\theta_0$, initial auxiliary vector $w_0\in\R^m$, step-size $\alpha>0$, parameter $\lambda\in[0,1]$, auxiliary step-size $\eta>0$ satisfying~\Cref{ass:gtd_style_inner_stepsize}, boundary-target step-size $\beta>0$, and inner iteration count $N\geq1$.
\FOR{target boundaries $n=0,1,2,\ldots$}
\STATE Select the deterministic greedy mode $\pi_n$ by~\Cref{eq:gtd_style_greedy_policy}.
\STATE Set $w_n^{(0)}=w_n$.
\FOR{$\ell=0,1,\ldots,N-1$}
\STATE Update $w_n^{(\ell+1)}$ by
\[
\begin{aligned}
w_n^{(\ell+1)}
&=
w_n^{(\ell)}+
\eta\alpha\Phi^\top D
\left(R+\gamma P\Pi^{\pi_n}\Phi\bar\theta_n-\Phi\bar\theta_n\right)\\
&\quad-
\eta\bigl[(1-\lambda)I+\lambda\alpha\Phi^\top D\Phi\bigr]w_n^{(\ell)}.
\end{aligned}
\]
\ENDFOR
\STATE Set $w_{n+1}=w_n^{(N)}$.
\STATE Set $\bar\theta_{n+1}=\bar\theta_n+\beta w_{n+1}$.
\ENDFOR
\end{algorithmic}
\end{algorithm}

With $\beta=1$ and sufficiently many inner iterations, this algorithm approaches
the direct inverse-form boundary update in~\Cref{alg:exact_deterministic_model_based_period_lambda_target}. For finite $N$,
it is an inverse-free approximation that keeps the same correction equation but
uses only Richardson subiterations.

\subsection{The One-Step Two-Line Case}
\label{sec:gtd_style_one_step_two_line_case}

The one-step inverse-free implementation is the $N=1$ case of~\Cref{alg:gtd_style_n_step_inner_richardson}.  Setting $w_n^{(0)}=w_n$, taking one
inner Richardson step, and then moving the boundary target gives the following
two-line recursion.

\setcounter{algorithm}{5}
\begin{algorithm}[H]
\caption{Inverse-Free Two-Line $\lambda$-DLQL Recursion}
\label{alg:gtd_style_two_line_inverse_free_recursion}
\begin{algorithmic}[1]
\REQUIRE Initial boundary target $\bar\theta_0$, initial auxiliary vector $w_0\in\R^m$, step-size $\alpha>0$, parameter $\lambda\in[0,1]$, auxiliary step-size $\eta>0$ satisfying~\Cref{ass:gtd_style_inner_stepsize}, and boundary-target step-size $\beta>0$.
\FOR{target boundaries $n=0,1,2,\ldots$}
\STATE Select the deterministic greedy mode $\pi_n$ by~\Cref{eq:gtd_style_greedy_policy}.
\STATE Update the auxiliary correction and then the boundary target:
\begin{minipage}[t]{0.94\linewidth}
\begin{align}
w_{n+1}
&=
w_n+
\eta\left[
\alpha\Phi^\top D
\left(
R+\gamma P\Pi^{\pi_n}\Phi\bar\theta_n-\Phi\bar\theta_n
\right)
-\bigl[(1-\lambda)I+\lambda\alpha\Phi^\top D\Phi\bigr]w_n
\right],
\label{eq:gtd_style_w_update}\\
\bar\theta_{n+1}
&=
\bar\theta_n+\beta w_{n+1}.
\label{eq:gtd_style_theta_update}
\end{align}
\end{minipage}
\ENDFOR
\end{algorithmic}
\end{algorithm}
\setcounter{algorithm}{6}

Here $w_n\in\R^m$ is the auxiliary correction variable, $\eta>0$ is the
auxiliary step-size, and $\beta>0$ is the boundary-target step-size in the second
line. The first line is a Richardson iteration for the correction equation in~\Cref{eq:period_lambda_linear_system}; the second line uses the current tracked
correction to move the boundary target. A stochastic-RL version can be obtained
by replacing the model-based residual in~\Cref{eq:gtd_style_w_update} with
sampled temporal-difference increments, as in standard stochastic-approximation
implementations of Q-learning~\citep{sutton1998reinforcement,tsitsiklis1994asynchronous}.

The case $N=1$ of~\Cref{alg:gtd_style_n_step_inner_richardson} is exactly~\Cref{alg:gtd_style_two_line_inverse_free_recursion}: $w_n^{(0)}=w_n$, one inner
Richardson step gives~\Cref{eq:gtd_style_w_update}, and the boundary move
$\bar\theta_{n+1}=\bar\theta_n+\beta w_{n+1}$ gives~\Cref{eq:gtd_style_theta_update}. Applying the second line with $\beta=1$ after
inner convergence would reproduce the inverse-form boundary update in~\Cref{alg:exact_deterministic_model_based_period_lambda_target}. In the actual
simultaneous recursion in~\Cref{eq:gtd_style_w_update}--\Cref{eq:gtd_style_theta_update}, however, the
auxiliary variable and the boundary target move together. Therefore the
inverse-free scheme tracks the same correction equation, but its switched
dynamics are not identical to the inverse-form boundary dynamics.

\begin{lemma}
\label{lem:gtd_style_limit_points_solve_pbe}
The two-line recursion does not change the equation solved at a limit point. If
$(\bar\theta_n,w_n)$ converges to $(\bar\theta_\infty,w_\infty)$ and $\beta>0$,
then $w_\infty=0$ and
\begin{equation}
\Phi^\top D
\left(
R+\gamma P\Pi^{\pi_\infty}\Phi\bar\theta_\infty
-\Phi\bar\theta_\infty
\right)=0,
\label{eq:gtd_style_limit_residual_greedy}
\end{equation}
where $\pi_\infty$ is a greedy selection for $\bar\theta_\infty$. Hence
$\Pi^{\pi_\infty}\Phi\bar\theta_\infty=V_{\bar\theta_\infty}$ and
\begin{equation}
\Phi^\top D
\left(
R+\gamma P V_{\bar\theta_\infty}-\Phi\bar\theta_\infty
\right)=0.
\label{eq:gtd_style_same_projected_bellman_equation}
\end{equation}
Consequently, whenever the projected Q-Bellman equation has the unique solution
$\theta^\star$, every convergent two-line trajectory has
$\bar\theta_\infty=\theta^\star$ and $w_\infty=0$. The auxiliary variable is zero
at the solution because the projected Bellman residual is zero there; it is only
used to track nonzero corrections away from the solution.
\end{lemma}
\begin{proof}
Taking limits in the second line
$\bar\theta_{n+1}=\bar\theta_n+\beta w_{n+1}$ gives
$0=\beta w_\infty$. Since $\beta>0$, we have $w_\infty=0$. Taking limits in the
first line and using $w_\infty=0$ gives~\Cref{eq:gtd_style_limit_residual_greedy} for a greedy selection
$\pi_\infty$ at $\bar\theta_\infty$. By the definition of the greedy value,
$\Pi^{\pi_\infty}\Phi\bar\theta_\infty=V_{\bar\theta_\infty}$, so~\Cref{eq:gtd_style_limit_residual_greedy} becomes~\Cref{eq:gtd_style_same_projected_bellman_equation}. If the projected
Q-Bellman equation has the unique solution $\theta^\star$, this limiting equation
forces $\bar\theta_\infty=\theta^\star$, and the already established identity
$w_\infty=0$ completes the proof.
\end{proof}

The convergence condition for this two-line recursion is not the original
$\lambda$-DLQL JSR condition, because the state is now the pair
$(\bar\theta_n,w_n)$.
\begin{lemma}
\label{lem:gtd_style_augmented_error_recursion}
Let $x_n:=\bar\theta_n-\theta^\star$.  For the fixed parameters $\alpha$, $\eta$,
and $\beta$ appearing in~\Cref{eq:gtd_style_w_update}--\Cref{eq:gtd_style_theta_update}, define the following policy-indexed augmented
matrix:
\[
\begin{aligned}
A_{\pi,\lambda}^{\mathrm{DLQL3}}
:={}&
\begin{bmatrix}
I+\beta\eta\alpha\Phi^\top D(\gamma P\Pi^\pi\Phi-\Phi)
&
\beta\left(I-\eta\bigl[(1-\lambda)I+\lambda\alpha\Phi^\top D\Phi\bigr]\right)\\
\eta\alpha\Phi^\top D(\gamma P\Pi^\pi\Phi-\Phi)
&
I-\eta\bigl[(1-\lambda)I+\lambda\alpha\Phi^\top D\Phi\bigr]
\end{bmatrix}.
\end{aligned}
\]
With a stochastic policy $\mu$ in place of $\pi$, the same block formula defines
$A_{\mu,\lambda}^{\mathrm{DLQL3}}$. Let
$\mu_n=\mu_{\bar\theta_n,\theta^\star}$ be the stochastic policy supplied
by~\Cref{lem:stochastic_policy_linearization}, so that
\[
V_{\bar\theta_n}-V_{\theta^\star}
=
\Pi^{\mu_n}\Phi x_n.
\]
Then the homogeneous recursion for the two-line scheme is
\begin{equation}
\begin{bmatrix}
x_{n+1}\\
w_{n+1}
\end{bmatrix}
=
A_{\mu_n,\lambda}^{\mathrm{DLQL3}}
\begin{bmatrix}
x_n\\
w_n
\end{bmatrix}.
\label{eq:gtd_style_augmented_error_recursion}
\end{equation}
\end{lemma}
\begin{proof}
The algorithmic policy $\pi_n$ in~\Cref{eq:gtd_style_w_update} is greedy for
$\bar\theta_n$, hence
$\Pi^{\pi_n}\Phi\bar\theta_n=V_{\bar\theta_n}$. Subtracting the projected
Q-Bellman fixed-point residual at $\theta^\star$ gives
\[
\begin{aligned}
&\Phi^\top D
\left(R+\gamma P\Pi^{\pi_n}\Phi\bar\theta_n-\Phi\bar\theta_n\right) \\
&\qquad
=\Phi^\top D\left(\gamma P(V_{\bar\theta_n}-V_{\theta^\star})-\Phi x_n\right)
=\Phi^\top D(\gamma P\Pi^{\mu_n}\Phi-\Phi)x_n.
\end{aligned}
\]
Therefore the first line of the two-line recursion gives
\[
w_{n+1}
=
\eta\alpha\Phi^\top D(\gamma P\Pi^{\mu_n}\Phi-\Phi)x_n
+
\left(I-\eta\bigl[(1-\lambda)I+\lambda\alpha\Phi^\top D\Phi\bigr]\right)w_n.
\]
The second line gives $x_{n+1}=x_n+\beta w_{n+1}$. Substituting the preceding
expression for $w_{n+1}$ into this identity yields the first block row of the
displayed matrix, while the expression for $w_{n+1}$ itself gives the second
block row, with $\pi$ replaced by $\mu_n$.
\end{proof}
For local analysis on a region where one deterministic policy mode represents
the value difference, write that mode as $\pi$.  For the global switching
certificate, one may use the same stochastic-policy linearization
in~\Cref{lem:stochastic_policy_linearization}; the algorithmic policy is still
the greedy policy in~\Cref{eq:gtd_style_greedy_policy}, but the proof represents
$V_{\bar\theta_n}-V_{\theta^\star}$ by a policy-indexed linear mode.  Thus a natural sufficient certificate for the inverse-free two-line
recursion is the JSR stability of the finite augmented family
\[
\mathcal A_{\lambda}^{\mathrm{DLQL3}}
:=
\set{A_{\pi,\lambda}^{\mathrm{DLQL3}}:\pi\in\Theta}.
\]
Equivalently, the augmented JSR condition is
\begin{equation}
\rho(\mathcal A_{\lambda}^{\mathrm{DLQL3}})<1,
\label{eq:gtd_style_augmented_jsr_condition}
\end{equation}
with the stochastic-policy version giving the same deterministic-mode certificate
by the convex-hull argument used for the earlier switching families.  This
condition is a stability condition for the augmented $2m$-dimensional system; it
is not the same as $\rho(\mathcal A_\lambda^{\mathrm{DLQL}})<1$ for the inverse-form
boundary update.

The augmented condition in~\Cref{eq:gtd_style_augmented_jsr_condition} is the
switched-system certificate needed for the simultaneous two-line recursion. The
next theorem shows that this certificate follows from the inverse-form JSR
condition itself, provided the boundary-target step-size $\beta$ is chosen
sufficiently small.
\begin{theorem}
\label{thm:gtd_style_jsr_small_beta_stability}
Fix $0\leq\lambda\leq1$ and suppose the auxiliary step-size satisfies~\Cref{ass:gtd_style_inner_stepsize}.
If either $0\leq\lambda<1$ and $\rho(\mathcal A_\lambda^{\mathrm{DLQL}})<1$, or
$\lambda=1$ and $\rho(\mathcal A^{\mathrm{PQVI}})<1$, then there exists
$\beta_0\in(0,1]$ such that, for every $0<\beta\leq\beta_0$, the corresponding
augmented family satisfies
\[
\rho(\mathcal A_{\lambda}^{\mathrm{DLQL3}})<1.
\]
Consequently, for every $0<\beta\leq\beta_0$, the inverse-free
two-line recursion is uniformly exponentially stable under arbitrary switching
whenever the corresponding inverse-form switched family is JSR-stable. The same
common contraction applies to the stochastic-policy augmented matrices obtained
by replacing $\pi$ with a stochastic policy $\mu$ satisfying
$\mu(s)\in\Delta_{|\mathcal A|}$. At $\beta=0$,
however, the corresponding augmented family has a unit slow block and
$\rho(\mathcal A_{\lambda}^{\mathrm{DLQL3}})=1$.
\end{theorem}
The proof is given in~\Cref{app:proof:thm:gtd_style_jsr_small_beta_stability}.
Thus the augmented JSR condition in~\Cref{eq:gtd_style_augmented_jsr_condition} for~\Cref{alg:gtd_style_two_line_inverse_free_recursion} can be inherited from the
original inverse-form JSR certificate only as a small-positive-$\beta$ result,
through one common Lyapunov norm for all switched modes.

The GTD-style deterministic two-line recursion removes the explicit inverse by
tracking the correction equation. If it converges, it converges to the same
projected Bellman solution as the original recursion, not to a regularized or
biased fixed point. The path and the stability certificate change:
the inverse-form method is governed by the $m$-dimensional boundary map, whereas~\Cref{eq:gtd_style_w_update}--\Cref{eq:gtd_style_theta_update} is governed by the
augmented $2m$-dimensional switching dynamics in~\Cref{eq:gtd_style_augmented_error_recursion}.

\section{Sampled Period-\texorpdfstring{$m$}{m} Implementation for \texorpdfstring{$\lambda$}{lambda}-DLQL}
\label{sec:sampled_period_m_targets}
\label{app:sampled_period_m_targets}

The $\lambda$-DLQL target in~\Cref{eq:lambda_averaged_target_parameter}
uses all finite hard-target periods with geometric weights. The recursive solver
in~\Cref{sec:recursive_period_lambda_solver} computes this average exactly from a
linear correction equation~\Cref{eq:period_lambda_linear_system}, while the
inverse-free recursion in~\Cref{alg:gtd_style_two_line_inverse_free_recursion}
tracks the same correction with the auxiliary update~\Cref{eq:gtd_style_w_update}--\Cref{eq:gtd_style_theta_update}. A second
implementation is to sample one period $m$ and use the sampled candidate as the
next boundary parameter. This section uses the sampled candidate directly as the
next boundary target.

For $0\leq\lambda<1$, let $M_\lambda$ be the geometric random variable on
$\{1,2,\ldots\}$ defined by
\begin{align}
\mathbb P(M_\lambda=m)=(1-\lambda)\lambda^{m-1},
\qquad m=1,2,\ldots.
\label{eq:sampled_m_geometric_law}
\end{align}
For $\lambda=0$, set $M_0=1$. The endpoint $\lambda=1$ is not sampled
from~\Cref{eq:sampled_m_geometric_law}; as in the preceding sections, it is the
continuous PQVI endpoint $A_\pi^{\mathrm{PQVI}}$.

For a fixed boundary target and deterministic mode, the sampled period-$m$
candidate is $h_\pi^m(\bar\theta;\bar\theta)$, as in~\Cref{eq:period_m_as_frozen_iterate}, and the corresponding boundary-error matrix
is $A_{\pi,m}^{\mathrm{DLQL}}$ from~\Cref{eq:m_dlql_A_mu_m_definition}. The sampled
implementation below uses exactly these existing objects.

\begin{algorithm}[H]
\caption{Single-Sample $\lambda$-DLQL Boundary Update}
\label{alg:sampled_m_period_lambda_target}
\begin{algorithmic}[1]
\REQUIRE Boundary target $\bar\theta_n$ and parameter $\lambda\in[0,1)$.
\STATE Freeze the boundary target $\bar\theta_n$.
\STATE Draw $M_n\sim\operatorname{Geom}(1-\lambda)$ on $\{1,2,\ldots\}$ using~\Cref{eq:sampled_m_geometric_law}; if $\lambda=0$, set $M_n=1$.
\STATE Set $\widehat\theta_n^{(0)}\leftarrow\bar\theta_n$.
\FOR{$i=0,\ldots,M_n-1$}
\STATE Update with the same frozen target:
$\displaystyle
\widehat\theta_n^{(i+1)}
=
\widehat\theta_n^{(i)}+
\alpha\Phi^\top D\left(R+\gamma P V_{\bar\theta_n}-\Phi\widehat\theta_n^{(i)}\right).
$
\ENDFOR
\STATE Set the sampled boundary candidate $\widehat\theta_{n+1}\leftarrow\widehat\theta_n^{(M_n)}$.
\STATE Set the next boundary parameter to the sampled candidate:
$\displaystyle
\bar\theta_{n+1}
\leftarrow
\widehat\theta_{n+1}.
$
\end{algorithmic}
\end{algorithm}

\Cref{alg:sampled_m_period_lambda_target} is the single-sample randomized period
update: the sampled candidate is not averaged with other candidates before the
boundary is moved. The next lemma identifies the mean target of this single
sample, using the same period maps and period-$\lambda$ matrices already
introduced above.

\begin{lemma}
\label{lem:sampled_m_unbiased_boundary_map}
Fix a boundary target $\bar\theta_n$ and a deterministic mode $\pi_n$. Let
$M_\lambda$ be independent of $\bar\theta_n$ and distributed according
to~\Cref{eq:sampled_m_geometric_law}. Then
\begin{align}
\mathbb E\left[
 h_{\pi_n}^{M_\lambda}(\bar\theta_n;\bar\theta_n)
\right]
=
\bar\theta_{n+1},
\label{eq:sampled_m_unbiased_parameter}
\end{align}
where $\bar\theta_{n+1}$ is the geometrically averaged target in~\Cref{eq:lambda_averaged_target_parameter}. In the homogeneous error model, for
any policy mode $\mu$ and every error vector $e$,
\begin{align}
\mathbb E\left[A_{\mu,M_\lambda}^{\mathrm{DLQL}}e\right]
=
A_{\mu,\lambda}^{\mathrm{DLQL}}e.
\label{eq:sampled_m_unbiased_error}
\end{align}
Consequently, if the boundary-error mode is $\mu_n$, the conditional mean of the
sampled update is
\begin{align}
\mathbb E\left[\bar\theta_{n+1}-\theta^\star\mid \bar\theta_n,\mu_n\right]
=
A_{\mu_n,\lambda}^{\mathrm{DLQL}}
(\bar\theta_n-\theta^\star).
\label{eq:sampled_m_soft_mean_update}
\end{align}
\end{lemma}

\begin{proof}
Using~\Cref{eq:sampled_m_geometric_law},
\[
\mathbb E\left[
 h_{\pi_n}^{M_\lambda}(\bar\theta_n;\bar\theta_n)
\right]
=
\sum_{m=1}^{\infty}(1-\lambda)\lambda^{m-1}
 h_{\pi_n}^{m}(\bar\theta_n;\bar\theta_n)
=
\bar\theta_{n+1}.
\]
This proves~\Cref{eq:sampled_m_unbiased_parameter}. The error identity is the
same calculation applied to the matrices $A_{\mu,m}^{\mathrm{DLQL}}$ that map one
target-boundary error to the next:
\[
\mathbb E\left[A_{\mu,M_\lambda}^{\mathrm{DLQL}}e\right]
=
\sum_{m=1}^{\infty}(1-\lambda)\lambda^{m-1}
A_{\mu,m}^{\mathrm{DLQL}}e
=
A_{\mu,\lambda}^{\mathrm{DLQL}}e.
\]
The sampled update satisfies
\[
\bar\theta_{n+1}-\theta^\star
=
A_{\mu_n,M_\lambda}^{\mathrm{DLQL}}(\bar\theta_n-\theta^\star),
\]
so taking the conditional expectation gives~\Cref{eq:sampled_m_soft_mean_update}.
\end{proof}

The preceding identity explains the average drift of the sampled update. A
uniform sampled-period second-moment bound is stated and proved in~\Cref{app:proof:prop:sampled_m_second_moment_bound}. It justifies square
integrability for the expectation estimate below; the contraction estimate itself
is pathwise.

The finite-time statement below states the stability condition directly in JSR
form. The relevant matrices are the boundary matrices produced after drawing a
period. We denote this sampled boundary family by
\begin{equation*}
\mathcal W_{\alpha}
:=
\set{A_{\pi,m}^{\mathrm{DLQL}}:
\pi\in\Theta,\ m\in\{1,2,\ldots\}}.
\end{equation*}
Thus once this family is JSR-stable, the bound holds for every realized sequence
of sampled periods. The theorem uses the same convex-hull passage as the earlier
switching certificates: the algorithm selects a greedy boundary target, while
the exact error representation may use the stochastic policy supplied
by~\Cref{lem:stochastic_policy_linearization}.
\begin{theorem}
\label{thm:sampled_m_soft_finite_time_bound}
Assume~\Cref{ass:hard_target_stepsize}. Fix $0\leq\lambda<1$, and let
$\theta^\star$ be the projected Bellman fixed point
from~\Cref{ass:unique_projected_q_bellman}. Let $\bar\theta_{n+1}$ be generated
by~\Cref{alg:sampled_m_period_lambda_target}. At each boundary, choose the deterministic greedy mode $\pi_n$ satisfying
$V_{\bar\theta_n}=\Pi^{\pi_n}\Phi\bar\theta_n$ with fixed tie breaking,
and represent the boundary error by
$\mu_n=\mu_{\bar\theta_n,\theta^\star}$ from~\Cref{lem:stochastic_policy_linearization}. Suppose that
\begin{equation*}
\rho(\mathcal W_{\alpha})<1.
\end{equation*}
Then, for every $\epsilon>0$ such that
\begin{equation*}
\rho(\mathcal W_{\alpha})+\epsilon<1,
\end{equation*}
there exists $C_\epsilon\geq1$ such that, for every $n\geq0$,
\begin{align}
\mathbb E\left[\|\bar\theta_n-\theta^\star\|_2^2\right]
\leq
C_\epsilon
\left[
\rho(\mathcal W_{\alpha})+\epsilon
\right]^{2n}
\|\bar\theta_0-\theta^\star\|_2^2.
\label{eq:sampled_m_soft_finite_time_bound}
\end{align}
The same estimate holds pathwise before taking expectation.
\end{theorem}

\begin{proof}
Under~\Cref{ass:hard_target_stepsize},
\(\rho(I-\alpha\Phi^\top D\Phi)<1\). Since \(\Phi^\top D\Phi\) is
symmetric positive definite, \(I-\alpha\Phi^\top D\Phi\) is symmetric and
\(\sup_{m\geq0}\|(I-\alpha\Phi^\top D\Phi)^m\|_2<\infty\). Because the
deterministic policy set \(\Theta\) is finite, the representation
\[
A_{\pi,m}^{\mathrm{DLQL}}
=A_\pi^{\mathrm{PQVI}}+(I-\alpha\Phi^\top D\Phi)^m(I-A_\pi^{\mathrm{PQVI}})
\]
from~\Cref{eq:m_dlql_converges_to_pqvi} shows that the family
$\mathcal W_{\alpha}$ is bounded.

For a realized boundary-error mode $\mu_n$ and sampled period $M_n$, the
hard-target candidate produced inside~\Cref{alg:sampled_m_period_lambda_target}
has homogeneous error
\[
\bar\theta_{n+1}-\theta^\star
=
A_{\mu_n,M_n}^{\mathrm{DLQL}}(\bar\theta_n-\theta^\star).
\]
For each fixed $M_n$, the dependence on $\Pi^{\mu_n}$ is affine, and the
stochastic-policy matrix lies in the convex hull of the corresponding
deterministic-policy matrices. Hence the realized matrix belongs to
$\co(\mathcal W_{\alpha})$.

The JSR is unchanged by passing to the convex hull, so
$\rho(\co(\mathcal W_{\alpha}))=\rho(\mathcal W_{\alpha})$. By the JSR definition
for the bounded family $\co(\mathcal W_{\alpha})$, there is a constant
$C_\epsilon\geq1$ such that every product $W_{n-1}\cdots W_0$ of matrices from
$\co(\mathcal W_{\alpha})$ satisfies
\[
\|W_{n-1}\cdots W_0\|_2^2
\leq
C_\epsilon\left[\rho(\mathcal W_{\alpha})+\epsilon\right]^{2n},
\qquad n\geq0.
\]
Applying this bound to the realized product along the sampled path gives
\[
\|\bar\theta_n-\theta^\star\|_2^2
\leq
C_\epsilon\left[\rho(\mathcal W_{\alpha})+\epsilon\right]^{2n}
\|\bar\theta_0-\theta^\star\|_2^2,
\qquad n\geq0.
\]
This is the claimed pathwise estimate. Taking expectations over the sampled
periods gives~\Cref{eq:sampled_m_soft_finite_time_bound}.
\end{proof}

\section{Conclusion}

We have presented the $\lambda$-DLQL target mechanism for linear Q-learning.
DLQL and PQVI solve the same projected Q-Bellman equation when the
fixed point is well defined, but they produce different switched error families.
The $\lambda$-DLQL boundary map connects these two endpoints by geometrically
averaging the hard-target period maps used as building blocks.

Under the relaxation step-size condition, the mode $A_{\pi,\lambda}^{\mathrm{DLQL}}$ equals the
DLQL mode at $\lambda=0$ and converges to the PQVI mode as
$\lambda\uparrow1$. The corresponding family
$\mathcal A_\lambda^{\mathrm{DLQL}}$ therefore gives a direct JSR certificate for the
boundary recursion. The recursive correction equation, its inverse-free auxiliary version, and the
sampled period-$m$ boundary update provide implementable forms of the same
$\lambda$-DLQL target mechanism without introducing a different fixed-point
equation.


\appendix

\section{Proof of~\Cref{thm:gtd_style_jsr_small_beta_stability}}
\label{app:proof:thm:gtd_style_jsr_small_beta_stability}

\begin{proof}
For a deterministic mode \(\pi\), the inverse-form boundary-error mode is
\[
A_{\pi,\lambda}^{\mathrm{DLQL2}}
=
I+\bigl[(1-\lambda)I+\lambda\alpha\Phi^\top D\Phi\bigr]^{-1}
\alpha\Phi^\top D(\gamma P\Pi^\pi\Phi-\Phi)
\]
for \(0\leq\lambda<1\). By~\Cref{lem:lambda_dlql2_equals_lambda_dlql}, this is the same matrix as
\(A_{\pi,\lambda}^{\mathrm{DLQL}}\). At the endpoint \(\lambda=1\), the same
inverse-form expression becomes
\[
I+(\Phi^\top D\Phi)^{-1}\Phi^\top D(\gamma P\Pi^\pi\Phi-\Phi)
=A_\pi^{\mathrm{PQVI}}.
\]
By~\Cref{lem:period_lambda_correction_matrix_invertible},
\((1-\lambda)I+\lambda\alpha\Phi^\top D\Phi\) is symmetric positive definite,
and~\Cref{ass:gtd_style_inner_stepsize} gives
\[
r:=
\left\|I-\eta\bigl[(1-\lambda)I+\lambda\alpha\Phi^\top D\Phi\bigr]\right\|_2<1.
\]

For \(0<\beta\leq1\), apply the same block similarity transformation to every
mode:
\[
\begin{bmatrix}z_n\\ w_n\end{bmatrix}
=
\begin{bmatrix}
I&
\beta\left(I-\eta\bigl[(1-\lambda)I+\lambda\alpha\Phi^\top D\Phi\bigr]\right)
\left(\eta\bigl[(1-\lambda)I+\lambda\alpha\Phi^\top D\Phi\bigr]\right)^{-1}\\
0&I
\end{bmatrix}
\begin{bmatrix}x_n\\ w_n\end{bmatrix},
\]
with inverse block matrix
\[
\begin{bmatrix}x_n\\ w_n\end{bmatrix}
=
\begin{bmatrix}
I&
-\beta\left(I-\eta\bigl[(1-\lambda)I+\lambda\alpha\Phi^\top D\Phi\bigr]\right)
\left(\eta\bigl[(1-\lambda)I+\lambda\alpha\Phi^\top D\Phi\bigr]\right)^{-1}\\
0&I
\end{bmatrix}
\begin{bmatrix}z_n\\ w_n\end{bmatrix}.
\]
Thus each augmented mode is transformed by the same \(2\times2\) block
similarity relation:
\[
\begin{aligned}
\begin{bmatrix}z_{n+1}\\ w_{n+1}\end{bmatrix}
&=
\begin{bmatrix}
I&
\beta\left(I-\eta\bigl[(1-\lambda)I+\lambda\alpha\Phi^\top D\Phi\bigr]\right)
\left(\eta\bigl[(1-\lambda)I+\lambda\alpha\Phi^\top D\Phi\bigr]\right)^{-1}\\
0&I
\end{bmatrix}
A_{\pi,\lambda}^{\mathrm{DLQL3}}                                      \\
&\quad\times
\begin{bmatrix}
I&
-\beta\left(I-\eta\bigl[(1-\lambda)I+\lambda\alpha\Phi^\top D\Phi\bigr]\right)
\left(\eta\bigl[(1-\lambda)I+\lambda\alpha\Phi^\top D\Phi\bigr]\right)^{-1}\\
0&I
\end{bmatrix}
\begin{bmatrix}z_n\\ w_n\end{bmatrix}.
\end{aligned}
\]
Expanding this \(2\times2\) block product gives the transformed recursion
\begin{align*}
z_{n+1}
&=\left(I+
\beta\bigl[(1-\lambda)I+\lambda\alpha\Phi^\top D\Phi\bigr]^{-1}
\alpha\Phi^\top D(\gamma P\Pi^\pi\Phi-\Phi)\right)z_n \\
&\quad-
\beta^2\bigl[(1-\lambda)I+\lambda\alpha\Phi^\top D\Phi\bigr]^{-1}
\alpha\Phi^\top D(\gamma P\Pi^\pi\Phi-\Phi)
\left(I-\eta\bigl[(1-\lambda)I+\lambda\alpha\Phi^\top D\Phi\bigr]\right) \\
&\qquad\times
\left(\eta\bigl[(1-\lambda)I+\lambda\alpha\Phi^\top D\Phi\bigr]\right)^{-1}w_n,\\
w_{n+1}
&=\eta\alpha\Phi^\top D(\gamma P\Pi^\pi\Phi-\Phi)z_n \\
&\quad+
\Bigl(I-\eta\bigl[(1-\lambda)I+\lambda\alpha\Phi^\top D\Phi\bigr]
-\beta\eta\alpha\Phi^\top D(\gamma P\Pi^\pi\Phi-\Phi) \\
&\qquad\times
\left(I-\eta\bigl[(1-\lambda)I+\lambda\alpha\Phi^\top D\Phi\bigr]\right)
\left(\eta\bigl[(1-\lambda)I+\lambda\alpha\Phi^\top D\Phi\bigr]\right)^{-1}\Bigr)w_n.
\end{align*}
The block identity uses only the commutation of
\(I-\eta[(1-\lambda)I+\lambda\alpha\Phi^\top D\Phi]\) with
\(\eta[(1-\lambda)I+\lambda\alpha\Phi^\top D\Phi]\) and
\[
\left(I+
\left(I-\eta\bigl[(1-\lambda)I+\lambda\alpha\Phi^\top D\Phi\bigr]\right)
\left(\eta\bigl[(1-\lambda)I+\lambda\alpha\Phi^\top D\Phi\bigr]\right)^{-1}\right)
\eta\bigl[(1-\lambda)I+\lambda\alpha\Phi^\top D\Phi\bigr]=I.
\]

Let \(p\) be the common Lyapunov norm from~\Cref{lem:common_lyapunov_construction}, applied to
\(\mathcal A_\lambda^{\mathrm{DLQL}}\) when \(0\leq\lambda<1\) and to
\(\mathcal A^{\mathrm{PQVI}}\) when \(\lambda=1\). Using~\Cref{lem:lambda_dlql2_equals_lambda_dlql} for \(0\leq\lambda<1\) and the PQVI
identity at \(\lambda=1\), there is a number \(\tau\in(0,1)\) such that
\[
p\left(A_{\pi,\lambda}^{\mathrm{DLQL2}}u\right)
\leq \tau p(u)
\]
for every deterministic mode \(\pi\) and every \(u\). Put \(c=1-\tau>0\). Since
\[
I+
\beta\bigl[(1-\lambda)I+\lambda\alpha\Phi^\top D\Phi\bigr]^{-1}
\alpha\Phi^\top D(\gamma P\Pi^\pi\Phi-\Phi)
\]
equals
\[
(1-\beta)I+
\beta A_{\pi,\lambda}^{\mathrm{DLQL2}},
\]
convexity of the norm gives
\[
p\left(\left[I+
\beta\bigl[(1-\lambda)I+\lambda\alpha\Phi^\top D\Phi\bigr]^{-1}
\alpha\Phi^\top D(\gamma P\Pi^\pi\Phi-\Phi)\right]u\right)
\leq (1-c\beta)p(u),
\qquad 0\leq\beta\leq1.
\]
Because the deterministic policy set is finite and all norms are equivalent,
there are finite constants \(a,b,d\geq0\), independent of \(\pi\), such that
\begin{align*}
&p\Bigl(
\bigl[(1-\lambda)I+\lambda\alpha\Phi^\top D\Phi\bigr]^{-1}
\alpha\Phi^\top D(\gamma P\Pi^\pi\Phi-\Phi)\\
&\qquad\qquad\times
\left(I-\eta\bigl[(1-\lambda)I+\lambda\alpha\Phi^\top D\Phi\bigr]\right)
\left(\eta\bigl[(1-\lambda)I+\lambda\alpha\Phi^\top D\Phi\bigr]\right)^{-1}v
\Bigr)\leq a\|v\|_2,\\
&\left\|\eta\alpha\Phi^\top D(\gamma P\Pi^\pi\Phi-\Phi)u\right\|_2
\leq b p(u),\\
&\Bigl\|
\eta\alpha\Phi^\top D(\gamma P\Pi^\pi\Phi-\Phi)\\
&\qquad\qquad\times
\left(I-\eta\bigl[(1-\lambda)I+\lambda\alpha\Phi^\top D\Phi\bigr]\right)
\left(\eta\bigl[(1-\lambda)I+\lambda\alpha\Phi^\top D\Phi\bigr]\right)^{-1}v
\Bigr\|_2
\leq d\|v\|_2
\end{align*}
for all \(u,v\) and all deterministic \(\pi\). Hence the transformed recursion
satisfies
\begin{align*}
p(z_{n+1})&\leq (1-c\beta)p(z_n)+\beta^2a\|w_n\|_2,\\
\|w_{n+1}\|_2&\leq b p(z_n)+(r+\beta d)\|w_n\|_2.
\end{align*}
Choose \(K\geq 2b/c\). Then choose \(\beta_0\in(0,1]\) small enough that, for
every \(0<\beta\leq\beta_0\),
\[
r+\beta(d+Ka)\leq 1-\frac{c\beta}{2}.
\]
For the weighted norm
\[
W_\beta(z,w)=\|w\|_2+\frac{K}{\beta}p(z),
\]
the preceding estimates give, uniformly over deterministic modes,
\[
W_\beta(z_{n+1},w_{n+1})
\leq
\left(1-\frac{c\beta}{2}\right)W_\beta(z_n,w_n),
\qquad 0<\beta\leq\beta_0.
\]
Thus every transformed deterministic mode is a contraction in the same norm
\(W_\beta\). Because the block similarity transformation is common to all modes,
the original family \(\mathcal A_{\lambda}^{\mathrm{DLQL3}}\) has JSR strictly
smaller than one.

For a stochastic policy \(\mu\), each occurrence of \(\Pi^\mu\) is a convex
combination of the corresponding deterministic policy matrices. The transformed
augmented matrix is therefore the same convex combination of the transformed
deterministic augmented matrices. Convexity of \(W_\beta\) gives the same
contraction for these stochastic-policy matrices.

Finally, at \(\beta=0\) each deterministic augmented matrix is block triangular,
\[
A_{\pi,\lambda}^{\mathrm{DLQL3}}
=
\begin{bmatrix}
I&0\\
\eta\alpha\Phi^\top D(\gamma P\Pi^\pi\Phi-\Phi)&
I-\eta\bigl[(1-\lambda)I+\lambda\alpha\Phi^\top D\Phi\bigr]
\end{bmatrix},
\]
so the slow block contributes the eigenvalue \(1\), while the lower-right block
has spectral radius strictly smaller than one. Hence
\(\rho(\mathcal A_{\lambda}^{\mathrm{DLQL3}})=1\), and the stability conclusion is
a small-positive-\(\beta\) statement.
\end{proof}
\section{Proofs for the Sampled Period-\texorpdfstring{$m$}{m} Implementation}
\label{app:sampled_period_m_proofs}

The sampled implementation in~\Cref{sec:sampled_period_m_targets} draws an
unbounded period. The main proof only needs the pathwise JSR contraction, but the
expectation statement is cleaner once the sampled period matrices are known to be
uniformly square-integrable. The following proposition records that bound using
the same period matrices as the main text.
\begin{proposition}
\label{app:proof:prop:sampled_m_second_moment_bound}
Assume~\Cref{ass:hard_target_stepsize}. There exists a finite constant
$C\geq1$ such that, for every deterministic policy $\pi\in\Theta$ and every
$m\geq1$,
\begin{equation}
\|A_{\pi,m}^{\mathrm{DLQL}}\|_2\leq C.
\label{eq:app_uniform_period_matrix_bound}
\end{equation}
Consequently, for every policy mode $\mu$, every $0\leq\lambda<1$, and every
error vector $e$,
\begin{equation}
\mathbb E\left[\|A_{\mu,M_\lambda}^{\mathrm{DLQL}}e\|_2^2\right]
\leq C^2\|e\|_2^2.
\label{eq:app_sampled_candidate_second_moment_bound}
\end{equation}
\end{proposition}

\begin{proof}
Under~\Cref{ass:hard_target_stepsize},
\(\rho(I-\alpha\Phi^\top D\Phi)<1\). Since \(\Phi^\top D\Phi\) is
symmetric positive definite, \(I-\alpha\Phi^\top D\Phi\) is symmetric and
\(\sup_{m\geq0}\|(I-\alpha\Phi^\top D\Phi)^m\|_2<\infty\). By~\Cref{eq:m_dlql_converges_to_pqvi},
\[
A_{\pi,m}^{\mathrm{DLQL}}
=A_\pi^{\mathrm{PQVI}}+(I-\alpha\Phi^\top D\Phi)^m(I-A_\pi^{\mathrm{PQVI}}).
\]
The deterministic policy set \(\Theta\) is finite, so
\(\max_{\pi\in\Theta}\|A_\pi^{\mathrm{PQVI}}\|_2<\infty\). Combining this finite
maximum with the uniform bound on \((I-\alpha\Phi^\top D\Phi)^m\) gives a
constant \(C\) satisfying~\Cref{eq:app_uniform_period_matrix_bound}.

For a stochastic policy mode $\mu$, the matrix $A_{\mu,m}^{\mathrm{DLQL}}$ is a
convex combination of the deterministic matrices $A_{\pi,m}^{\mathrm{DLQL}}$.
Thus the same norm bound holds for $A_{\mu,m}^{\mathrm{DLQL}}$. Taking the
expectation over $M_\lambda$ gives~\Cref{eq:app_sampled_candidate_second_moment_bound}.
\end{proof}

\end{document}